\documentclass[journal]{IEEEtran}

\usepackage{cite}
\usepackage[breaklinks=true,colorlinks=true,bookmarksopen,bookmarksnumbered,citecolor=blue,urlcolor=blue,linkcolor=blue,pdfauthor={Chen Wang, Xu Fang, Thien-Minh Nguyen, Lihua Xie},pdftitle={Graph Optimization Approach to Localization with Range Measurements}]{hyperref}

\usepackage{balance}

\usepackage{booktabs}
\usepackage{float}

\ifCLASSINFOpdf
   \usepackage[pdftex]{graphicx}
   \DeclareGraphicsExtensions{.pdf,.jpeg,.png}
\else
 
\fi

\usepackage{mathrsfs,amsmath}
\usepackage{amssymb}
\usepackage{siunitx}
\usepackage{xcolor}

\usepackage{algorithm}
\usepackage{algpseudocode}
\usepackage{enumitem}

\usepackage{amsthm,hyperref}
\newtheorem{theorem}{Theorem}
\newtheorem{remark}{Remark}

\usepackage{mathtools}

\DeclarePairedDelimiter{\diagfences}{(}{)}
\newcommand{\diag}{\operatorname{diag}\diagfences}

\begin{document}

\title{Graph Optimization Approach to Range-based  Localization}

\author{Xu Fang, Chen Wang, Thien-Minh Nguyen, Lihua Xie
\thanks{Xu Fang, Thien-Minh Nguyen, and Lihua Xie are with the School of Electrical and Electronic Engineering, Nanyang Technological University, 639798, Singapore. (e-mail: fa0001xu@e.ntu.edu.sg; e150040@e.ntu.edu.sg; elhxie@ntu.edu.sg).} 
\thanks{Chen Wang is with the Robotics Institute, Carnegie Mellon University, Pittsburgh, PA 15213, USA.  (e-mail: chenwang@dr.com).} 
}

\maketitle

\begin{abstract}
    In this paper, we propose a general graph optimization based framework for
    localization, which can accommodate different types of measurements with varying measurement time intervals. Special emphasis will be on range-based localization.
	Range and trajectory smoothness constraints are constructed in a position graph, then the robot trajectory over a sliding window is estimated by a graph based optimization algorithm.
	Moreover, convergence analysis of the algorithm is provided, and the effects of the number of iterations and window size in the optimization on the localization accuracy are analyzed.
	{Extensive experiments on quadcopter under a variety of scenarios} verify the effectiveness of the proposed algorithm and demonstrate a much higher localization accuracy than the existing range-based localization methods, especially in the altitude direction. 
\end{abstract}


\IEEEpeerreviewmaketitle

\section{Introduction}

\IEEEPARstart{A}{ccurate}, efficient and reliable localization plays important roles in real-time robot-related applications \cite{minaeian2016vision, wu2018passive, wang2017non, wang2018kernel,zhu2017survey,li2019interval} such as formation, swarming and target search.
The vision-based simultaneous localization and mapping (SLAM)  technologies  \cite{JakobEngel:2014uc} have an unacceptable drift over a long run without odometer correction and need significant computational resources for dense mapping, which is not suitable for ultra-low power processors. The WiFi-based localization \cite{Chen:2015bk} has the problem of estimation fluctuations caused by the variation of signals and its low localization accuracy makes it inapplicable to robots such as unmanned aerial vehicles (UAV). Optical motion capture systems can provide millimeter level of localization accuracy \cite{Lupashin:2014bf}, but they are very expensive and confined to limited space.

An alternative method which utilizes the ultra-wideband (UWB) technology \cite{zihajehzadeh2017novel} has attracted researchers' attention due to its robustness to multipath and non-line of sight effects.
The UWB modules with known positions are referred to anchors. Robots carrying UWB modules are able to exchange information and calculate their distances to anchors by measuring the time of flight of signal, which are then used for estimating their own positions. 

However, there are limitations of existing range-based localization algorithms.
First, many algorithms \cite{Cui:2016bg,YiShang:2003kx,Joao:2003wa} such as multilateration and multidimensional scaling (MDS)  algorithms leveraging on optimization require that the mobile robot receive multiple concurrent range measurements. They may have relatively low localization accuracy when the range sensors cannot support multi-channels. For example, the UWB sensors usually use a single wireless channel. 
The neglect of minor time difference between consecutive measurements  brings localization error for the mobile robot.

Second, some algorithms consider such minor time difference, but need an accurate kinematic model. The representative examples are moving horizon estimation (MHE) \cite{alessandri2017fast} and the filter-based methods such as extended Kalman filter (EKF) \cite{Ledergerber:2015ur, mueller2015fusing, Guo:2016ff,  chen2016smartphone}. However, an accurate kinematic model may be hard to obtain due to the complex structure of robots, and a simplified or linearized kinematic model degrades their localization performance.  Last, the recent trend towards machine learning based methods stimulates a new wave of research, but generally it is still difficult to achieve good performance in real-time \cite{VanNguyen:2015du, MohammadAbu:2015tk, wang2019kervolutional}. 
These challenges open space for accurate, reliable, and robust localization techniques.

From the experimental results in the existing range-based localization methods \cite{Guo:2016ff,Anonymous:tWWRLPVI, Cui:2016bg,cui2018received}, we find that their performance in the altitude direction is generally not as good as other directions. Some possible solutions include: $(\romannumeral1)$ Adding altitude sensors such as Laser beam or Lidar to measure the altitude, but it requires the ground to be even;  $(\romannumeral2)$ Placing anchors on the ceiling, but it may also be difficult for many environments.

The graph optimization approach was originated from the vision-based SLAM technology \cite{JakobEngel:2014uc,kummerle2011g}. By using this technique, we shall present a general graph optimization based framework for localization, which can accommodate different kinds of measurements with varying measurement time intervals. Special emphasis will be on range-based localization, which mitigates the requirements of an accurate kinematic model, multi-channel support and high power processors, and
improves the localization performance.  It is worth noting that the existing geometry optimization methods \cite{Lin:2015kc,  DiFranco:2017ib, li2019globally} are suitable for estimating a static sensor network, which may not applicable for localizing a mobile robot.

The proposed range-based localization
jointly imposes the range and trajectory smoothness constraints over a sliding trajectory window that has several characteristics: $(\romannumeral1)$ It removes the dependence on kinematic model and the requirement of receiving concurrent multiple range measurements; $(\romannumeral2)$  It estimates the trajectory over a window instead of single position estimation, and can be implemented real-time in some low power systems; $(\romannumeral3)$ It is robust to outliers due to fusion with an outlier rejection algorithm; $(\romannumeral4)$ From the experimental results, it is observed that the localization accuracy is much improved, especially in the altitude direction.  This work is based on our previous works \cite{wang2017ultra,fang2018model}.
The main contributions of this paper are summarized as: 
\begin{enumerate}
\item A general localization framework based on graph optimization approach is proposed, which can accommodate different kinds of measurements with varying measurement time intervals. Special emphasis will be on range-only based localization and range-orientation based localization
\item Stability analysis of the algorithm is provided, and the effects of the number of iterations and window size in the optimization on the localization accuracy are analyzed.

\item  {The experimental results \url{https://youtu.be/UuMBSrCEs6Q} on quadcopter demonstrate its stability as well as much higher localization accuracy} than existing range-based algorithms, especially in the altitude direction without the need of placing anchors on the ceiling or adding altitude sensors. 
\end{enumerate}

This paper is organized as follows: 
Section \ref{sec:preliminary} proposes a general framework for localization. The basic concepts, problem description, problem formulation and special emphases on range-only based localization and range-orientation based localization are presented, respectively. In Section \ref{a6},  the optimization algorithm and computational complexity analysis are provided. Section \ref{a5} presents the stability analysis. Section \ref{a7}  provides the details of experimental results. Section \ref{a9} ends this paper with conclusions.

\begin{figure}[!t]
\centering
\includegraphics[width=0.6\linewidth]{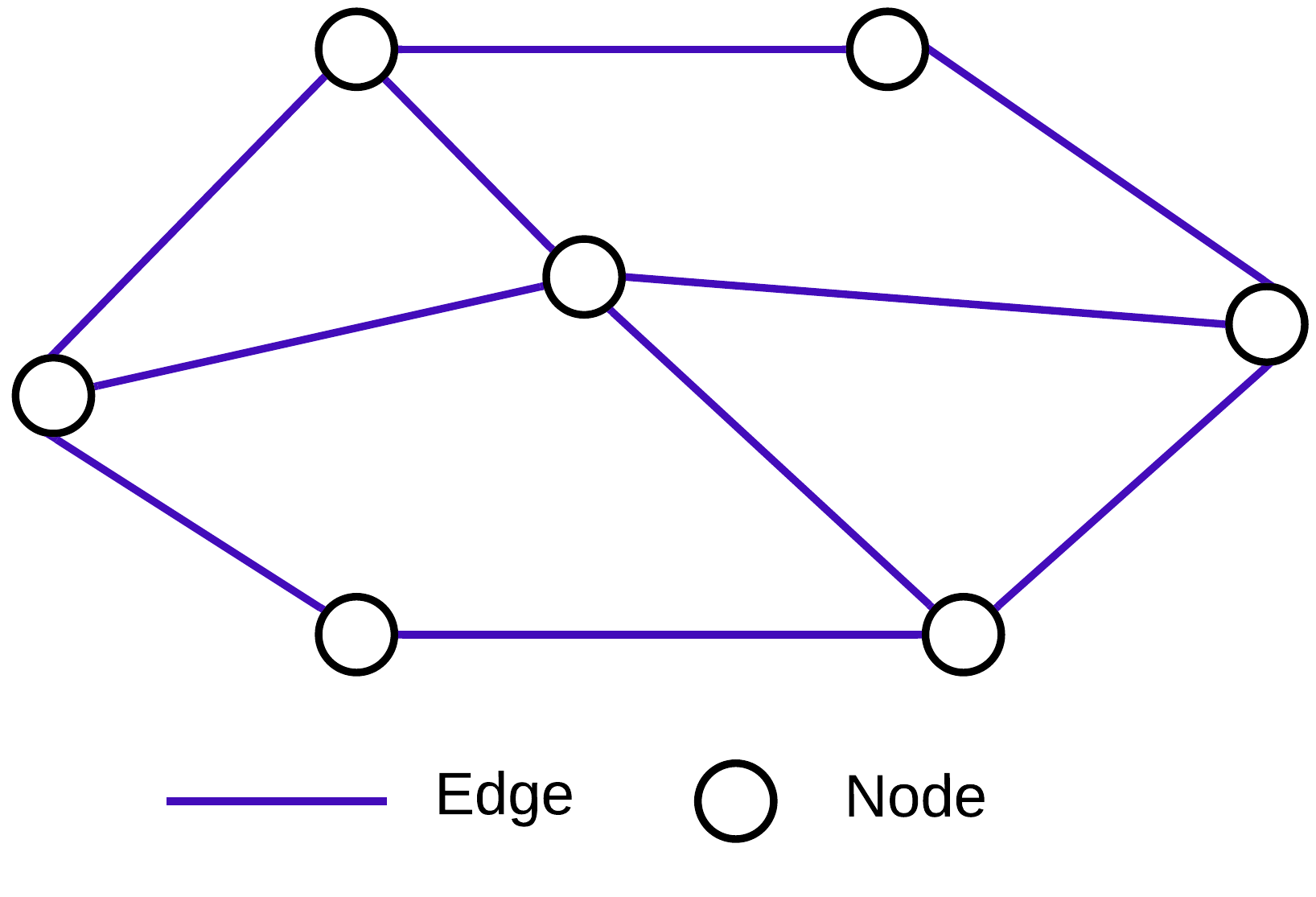}
\caption{A general framework for localization.
} 
\label{figg}
\end{figure}

\section{{A general framework for localization}}\label{sec:preliminary}

\subsection{Basic Concepts and Problem Description}\label{basic}

Considering the robot motion in 3-D space, its pose can be represented by a transformation matrix $\mathbf{P}$: 
\begin{equation}\label{1}
\mathbf{P} = \left[\begin{array}{cc}
    \mathbf{R}  & \mathbf{t} \\
    \mathbf{0}  &  1 
    \end{array}\right],\ \
    \mathbf{R} \in \mathbb{R}^{3 \times 3}, \ \ \mathbf{t} \in \mathbb{R}^{3},
\end{equation} 
where $\mathbf{R}$ is the rotation matrix, and $\mathbf{t}$ is the translation vector. The rotation matrix $\mathbf{R}$ and transformation matrix $\mathbf{P}$ belong to the Lie groups $\text{SO(3)}$ and $\text{SE(3)}$ \cite{Brian:2015vo}, respectively.
The rotation matrix $\mathbf{R}$ and transformation matrix $\mathbf{P}$ can be represented by vectors $\omega$ and
$\epsilon$ via an exponential mapping, i.e,
\begin{equation}\label{5}
\begin{array}{ll}
&\epsilon = ( \mathbf{\omega},\mathbf{u})^T \in \mathbb{R}^6, \omega = (\omega_1,\omega_2,\omega_3)^{T} \in \mathbb{R}^3 , \mathbf{u}^T \in \mathbb{R}^3, \\
&\omega_{\times}  \!=\! \left[\begin{array}{ccc}
    0 & -\omega_3 &\omega_2 \\
    \omega_3  &  0 & -\omega_1 \\
    -\omega_2 &\omega_1 &0 
    \end{array}\right], \ \ \mathbf{\epsilon_\times} = \left[\begin{array}{cc}
    \mathbf{\omega_{\times}}  & \mathbf{u}^T \\
    \mathbf{0}  &  0 
    \end{array}\right], \\ 
    & \mathbf{R} = \exp(\omega_{\times}), \ \
    \mathbf{P} = \exp(\epsilon_{\times}).
\end{array} 
\end{equation}

The rotation matrix $\mathbf{R}$ and transformation matrix $\mathbf{P}$ span over non-Euclidean spaces. Based on equation \eqref{5}, the functions 
$\mathbf{log}_\text{SO(3)}(\cdot)$ and $\mathbf{log}_\text{SE(3)}(\cdot)$ are defined such that the 
rotation matrix $\mathbf{R}$ and transformation matrix $\mathbf{P}$ in non-Euclidean spaces are mapped into their corresponding Euclidean spaces.
\begin{equation}\label{21}
\mathbf{log_{\text{SO(3)}}(\mathbf{R}) = \mathbf{\omega}}, \ \ \mathbf {log_{\text{SE(3)}}(\mathbf{P}) = \epsilon}.
\end{equation}

Then, the conventional optimization methods such as Levenberg-Marquardt method applicable to a Euclidean space can be used for rotation estimation and transformation estimation. 

The graph $G=\{ \mathcal{V},\mathcal{E}\}$ consisting of nodes and edges in Fig. \ref{figg}
shows the structure of the proposed localization, where $\mathcal{V}=\{\nu_1, \nu_2, \cdots, \nu_n \}$ is the node set, and $\mathcal{E}  \subseteq \mathcal{V} \times \mathcal{V}$ is the edge set. Denoted by $(\nu_i,\nu_j) \subseteq \mathcal{E}$ an edge of $G$. Each edge represents a constraint between two nodes.
The graph $G$
allows fusion of different measurements by imposing constraints on the nodes. The problem is how to localize the nodes based on the constraints. For example, if each node in the graph $G$ represents a mobile robot, 
some robots are equipped with Global Positioning System (GPS) and know their own positions, while the rest do not know their own positions.
The problem becomes how to localize the rest robots based on the known positions of some robots and the constraints among the robots. 

Each edge $(\nu_i,\nu_j)$ connects two nodes $\nu_i,\nu_j$, and the translations, rotations and transformations of the two nodes $\nu_i,\nu_j$ are denoted by $\mathbf{t}_i, \mathbf{t}_j \in \mathbb{R}^{3}$ , $\mathbf{R}_i, \mathbf{R}_j \in \mathbb{R}^{3 \times 3}$ and $\mathbf{P}_i, \mathbf{P}_j \in \mathbb{P}^{4 \times 4}$, respectively. There are four kinds of constraints.

\subsubsection{Range constraint}
Denoted by $d_{ij} \in \mathbb{R}$ the range measurement between nodes $\nu_i,\nu_j$. The range constrained equation is defined as 
\begin{subequations}\label{et1}
\begin{align}
& {E_{r}(\mathbf{t}_i,\mathbf{t}_j)}  =  w_{r}^{ij} \cdot \rho( {e_{r}^{ij}} ),\\
& e_{r}^{ij}  =  d_{ij} - ||\mathbf {t}_{i}-\mathbf{t}_{j}||_2,
\end{align}
\end{subequations}
where  
$w_{r}^{ij} \in \mathbb{R}$ is the weight. $||\cdot||_2$ is the Euclidean norm of a vector or the spectral norm of a matrix. $\rho(\cdot)$ is the Pseudo-Huber loss function defined as $\rho(\varrho) = \xi^2(\sqrt{1+(\varrho/\xi)^2}-1)$ where $\xi>0$ is the slope parameter.

\subsubsection{Relative translation constraint}

Denoted by $\mathbf{l}_{ij} \in \mathbb{R}^{3}$ the relative translation measurement between nodes $\nu_i,\nu_j$. The relative translation constrained equation is designed as
\begin{subequations}\label{et3}
\begin{align}
& {E_{t}(\mathbf{t}_i,\mathbf{t}_j)}  =   \rho \left( \sqrt{ {\mathbf{e}_{t}^{ij}}^T\mathbf{w}_{t}^{ij}\mathbf{e}_{t}^{ij} } \right) ,\\
& \mathbf{e}_{t}^{ij}  =  
\mathbf{l}_{ij} - (\mathbf {t}_{i}-\mathbf{t}_{j}),
\end{align}
\end{subequations}
where $\mathbf{w}_{t}^{ij} \in \mathbb{R}^{3 \times  3}$ is the weight. 

\subsubsection{Relative rotation constraint}

{Denoted by $\mathbf{u}_{ij} \in \mathbb{R}^{3 \times  3}$ the relative rotation measurement between nodes $\nu_i,\nu_j$. The relative rotation
constrained equation is designed as}
\begin{subequations}\label{et4}
\begin{align}
& {E_{o}(\mathbf{R}_i,\mathbf{R}_j)}  =   \rho\left( \sqrt{ {\mathbf{e}_{o}^{ij}}^T\mathbf{w}_{o}^{ij}\mathbf{e}_{o}^{ij} } \right),\\
& \mathbf{e}_{o}^{ij}  =  \mathbf {log_{\text{SO(3)}}(\mathbf{u}_{ij} \cdot (\mathbf {R}_{i}\mathbf{R}_{j})^{-1})},
\end{align}
\end{subequations}
where $\mathbf{w}_{o}^{ij} \in \mathbb{R}^{3 \times  3}$ is the weight, and function $\mathbf{ log_{\text{SO(3)}}(\cdot)}$ is defined in \eqref{21}. 

\subsubsection{Relative transformation constraint}

{Denoted by $\mathbf{q}_{ij} \in \mathbb{R}^{4 \times  4}$ the relative transformation measurement between nodes $\nu_i,\nu_j$. The relative transformation constrained equation is designed as}
\begin{subequations}\label{et5}
\begin{align}
& {E_{p}(\mathbf{P}_i,\mathbf{P}_j)}  =   \rho\left( \sqrt{ {\mathbf{e}_{p}^{ij}}^T\mathbf{w}_{p}^{ij}\mathbf{e}_{p}^{ij} } \right),\\
& \mathbf{e}_{p}^{ij}  =   \mathbf{ log_{\text{SE(3)}}}(
\mathbf{q}_{ij} \cdot  (\mathbf{P}_{i}\mathbf{P}_{j})^{-1}),
\end{align}
\end{subequations}
where $\mathbf{w}_{p}^{ij} \in \mathbb{R}^{6 \times  6}$ is the weight, and function $\mathbf{ log_{\text{SE(3)}}(\cdot)}$ is defined in \eqref{21}.

\subsection{Problem Formulation}

{For the node set $\mathcal{V}=\{\nu_1, \nu_2, \cdots, \nu_n \}$, the translations $\mathbf{t}=(\mathbf{t}_1^T,\mathbf{t}_2^T,\cdots,\mathbf{t}_n^T)^T$ and  
transformations  $\mathbf{P}=(\mathbf{P}_1^T,\mathbf{P}_2^T,\cdots,\mathbf{P}_n^T)^T$ of the nodes can be estimated respectively by solving the following cost functions \eqref{g1} and \eqref{g2}.}
\begin{subequations}\label{g1}
\begin{align}
& F( \mathbf{t}) = \sum\limits_{(\nu_i,\nu_j) \subseteq \mathcal{E}}(E_{r}^{ij} +E_{t}^{ij}), \\
& \mathbf {\hat t}  = {\arg}\min  F( \mathbf{ t}),
\end{align}
\end{subequations}
where $\mathbf {\hat t}=(\mathbf{\hat t}_1^T,\mathbf{\hat t}_2^T,\cdots,\mathbf{\hat t}_n^T)^T$ are the estimated translations.
\begin{subequations}\label{g2}
\begin{align}
& F( \mathbf{P}) = \sum\limits_{(\nu_i,\nu_j) \subseteq \mathcal{E}}(E_{r}^{ij} +E_{t}^{ij}+E_{o}^{ij}+E_{p}^{ij}), \\
& \mathbf {\hat P}  = {\arg}\min  F( \mathbf{ P}),
\end{align}
\end{subequations}
where $\mathbf {\hat P}=(\mathbf{\hat P}_1^T,\mathbf{\hat P}_2^T,\cdots,\mathbf{\hat P}_n^T)^T$ are estimated transformations.

\begin{remark}
For this general framework, we provide a corresponding application platform  \url{https://github.com/wang-chen/localization}, which
can accommodate different kinds of measurements for localization. The proposed general framework is inspired by the vision-based SLAM technology \cite{JakobEngel:2014uc,kummerle2011g}. {In the SLAM technology, vision measurement is indispensable,
but in our framework, the mobile robot can be localized without vision measurement.}
\end{remark}

\begin{remark}\label{r4}
The weights  ${w}_{r}^{ij}$,  $\mathbf{w}_{t}^{ij}$, $\mathbf{w}_{o}^{ij}$ and $\mathbf{w}_{p}^{ij}$ in \eqref{et1}-\eqref{et5} can be chosen based on their measurement accuracy. {Denote the covariance of a measurement noise by $\sigma^2$. The most straightforward way is to set $\mathbf{w}$ = $\frac{1}{\sigma^2+1}$. Note that the $1$ in the denominator helps prevent singularity when $\sigma^2$ is very small.}
\end{remark}

Based on the proposed general framework, in this paper, {special emphasis will be on range-only based localization and range-orientation based localization.} The structure of the range-based localization
consisting of several fixed anchors and a mobile robot is shown in Fig. \ref{fig2}. For this special structure,
the constraints can be divided into two categories: 
\begin{enumerate}
    \item Range constraint: the constraint between the robot and an anchor;
    \item Trajectory smoothness constraint: the constraint between adjacent robot translations.
\end{enumerate}

\begin{figure*}[!t]
\centering
\includegraphics[width=0.9\linewidth]{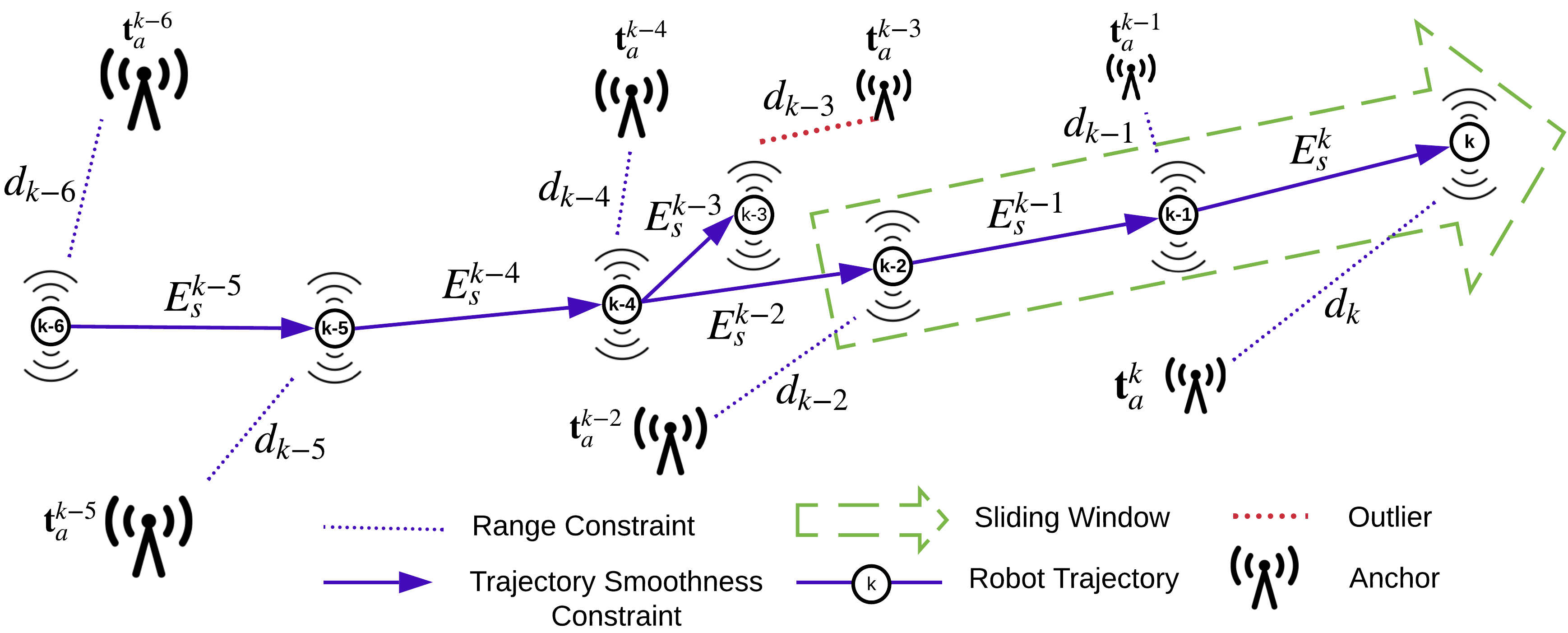}
\caption{The proposed range-based 
localization framework is shown in this figure.  Each edge represents a constrained equation. At each time instant $k$, we construct a range constrained equation $E_{r}^{k}$ and a trajectory smoothness constrained equation $E_{s}^{k}$. If the robot received a measurement outlier, the corresponding range constrained equation such as $E_{r}^{k\!-\!3}$ and trajectory smoothness constrained equation such as $E_{s}^{k\!-\!3}$
will not be added to the cost function.
} 
\label{fig2}
\end{figure*}

\subsection{Range-only based Localization}\label{a44}

In this part, we focus on a range-only based localization. Compared with the EKF-based or MHE-based range localization algorithms \cite{mueller2015fusing,   Benini:2012br,
pillonetto2010unconstrained,kimura2014vehicle,simonetto2011distributed,wang2014optimization,girrbach2017optimization, Guo:2016ff} which use several kinds of measurements for localization shown in Table. \ref{tablem}, the proposed method only needs range measurements.

\begin{table}[t]
  \begin{center}
  \caption{ {Comparison with existing EKF-based or MHE-based range localization methods.}}
  \resizebox{1.\columnwidth}{!}{
        \begin{tabular}{ccc}
        \toprule
         \textbf{Methods} & \textbf{Measurements}  & \textbf{Localization}  \\
        \midrule
        EKF \cite{mueller2015fusing} &
        Range, Acceleration, Angular rate & 2-D,3-D
        \\
        EKF \cite{Benini:2012br} & Range, Orientation, Odometry & 2-D \\
        EKF \cite{Guo:2016ff} & Range, Acceleration & 2-D \\
        MHE \cite{pillonetto2010unconstrained} & Range, Translation Speed, Rotational Speed & 2-D \\
        MHE \cite{kimura2014vehicle}  & Range, Velocity, Odometry & 2-D  \\
        MHE \cite{simonetto2011distributed}  & Range, Translation Speed, Rotational Speed & 2-D  \\
        MHE \cite{wang2014optimization}  & Range, Orientation, Altitude & 3-D  \\
            MHE \cite{girrbach2017optimization}  & Range, Acceleration, Angular rate, GNSS & 3-D  \\
        \textbf{Our method} & Range & 2-D,3-D\\
        \bottomrule
        \end{tabular}}
        \label{tablem}
  \end{center}
    \textbf{GNSS}: Global navigation satellite system \\
    \textbf{Our method}: The proposed range-only based localization
  \end{table}

\subsubsection{Range constrained equation}
The robot translation at time instant $k$ is denoted by $\mathbf{t}_k \in \mathbb{R}^{3}$, so that the robot trajectory can be denoted as $\mathbf{t}=(\mathbf{t}_1^T,\mathbf{t}_2^T,\cdots,\mathbf{t}_k^T)^T$. Our UWB ranging algorithm uses two-way time of flight measurement to calculate the range. A mobile robot equipped with a UWB module at translation $\mathbf{t}_{k}$ is able to range to one of the fixed UWB anchors at translation $\mathbf{t}_a^{k}$ shown in Fig. \ref{fig2}. The corresponding range measurement between $\mathbf{t}_{k}$ and $\mathbf{t}_a^{k}$ is denoted by $d_k$, which is obtained by the multiplication of light speed $c$ and the measurement of time of flight:
\begin{equation}\label{25}
\begin{array}{c} 
d_k=c\frac{Q_{ks}-Q_{kr}}{2}  + \eta_{k}= ||\mathbf {t}_{k}-\mathbf{t}_a^{k}||_2 + \eta_{k}, \\
\end{array}
\end{equation}
where 
$Q_{ks}$ and $Q_{kr}$ are the synchronized time instants when the UWB ranging radio is sent and received relative to the robot's clock respectively. $\eta_k$ is the bounded range measurement noise with $|\eta_k| \le \eta$. By applying the 3-$\sigma$ rule under the assumption that the range measurement noise $\eta_k$ in \eqref{25} follows an approximately normal distribution, we get $\eta_k\sim\mathbb{N}(0,\sigma_r^2)$ with variance $\sigma_r^2=\frac{\eta^2}{9}$.
At time instant $k$, the range constrained equation ${E_{r}^{k}}$ is defined as
\begin{subequations}\label{26}
\begin{align}
{E_{r}^{k}}  &=  w_{r}^{k} \cdot \rho( {e_{r}^{k}} ),\\
e_{r}^{k} & =  d_{k} - ||\mathbf {t}_{k}-\mathbf{t}_a^{k}||_2,
\end{align}
\end{subequations}
where 
$w_{r}^{k}$ is the weight given by
\begin{equation}\label{fa}
w_{r}^{k} = \frac{\iota^{2}}{\sigma_{r}^2 + \iota^{2}},   
\end{equation}
where $\iota$ is a free-parameter. 

\begin{remark}
From Remark \ref{r4}, the weight can be set as ${w}_{r}^k$ = $\frac{1}{\sigma_r^2+1}$. To be more flexible, we add another free parameter $\iota$ to shrink $\sigma_r$. Then we have ${w}_{r}^k = \frac{1}{\sigma_r^2/{\iota^2}+1}= \frac{\iota^2}{\sigma_r^2+\iota^2}$, where the square is to make it positive. 
\end{remark}

\subsubsection{Trajectory smoothness constrained equation}

Equation \eqref{26} only gives constraints on a set of sparse points and fails to form a smooth trajectory. To solve this problem, a trajectory smoothness constrained equation between adjacent robot translations is needed. The moving equation between adjacent translations $\mathbf{t}_k$ and $\mathbf{t}_{k-1}$  is
\begin{equation}\label{m1}
    \mathbf{t}_k = \mathbf{t}_{k-1}+ \mathbf{r}_k,
\end{equation}
where $\mathbf{r}_k$ is the relative translation between translations $\mathbf{t}_k$ and $\mathbf{t}_{k-1}$. One issue is that the relative translation $\mathbf{r}_k$ is unknown,
but we have 
\begin{equation}\label{m2}
    ||\mathbf{r}_k||_2 = ||\mathbf{t}_k-\mathbf{t}_{k-1}||_2 \le v_{\text{max}} \cdot \Delta T_k,
\end{equation}
where ${v_{\text{max}}}$ is the maximum velocity of the robot. $\Delta T_k$ is the time interval between translations $\mathbf{t}_k$ and $\mathbf{t}_{k-1}$. 
Similarly, 
by applying the 3-$\sigma$ rule under the assumption that the relative distance $||\mathbf{r}_k||_2$ follows an approximately normal distribution, we get $||\mathbf{r}_k||_2\sim\mathbb{N}(0,(\sigma_s^k)^2)$ with variance $(\sigma_s^k)^2=\frac{(v_{\text{max}}\cdot \Delta T_k)^2}{9}$.  Then, a trajectory smoothness constrained equation ${E_{s}^{k}}$ between adjacent translations
$\mathbf{t}_{k}$ and $\mathbf{t}_{{k-1}}$ is defined as
\begin{subequations}\label{27}
\begin{align}
{E_{s}^{k}} &=  w_{s}^{k} \cdot \rho( {e_{s}^{k}}),\\
e_{s}^{k} &= ||\mathbf {t}_{k}-\mathbf{t}_{{k-1}}||_2,
\end{align}
\end{subequations}
where $w_s^{k}$ is the weight designed as
\begin{equation}\label{d1}
w_s^{k} = \frac{\iota^{2}}{(\sigma_s^k)^2 + \iota^{2}},
\end{equation}
where $\iota$ is a free-parameter.

\subsubsection{Cost function}
At any time instant $k$ when a range measurement $d_k$ is received, we construct a range constrained equation $E_{r}^{k}$ \eqref{26} and a trajectory smoothness constrained equation $E_{s}^{k}$ \eqref{27}, thus a trajectory $\mathbf{t}=(\mathbf{t}_1^T,\mathbf{t}_2^T,\cdots,\mathbf{t}_k^T)^T$ can be estimated by solving the following cost function:
\begin{subequations}\label{c1}
\begin{align}
F( \mathbf{t}) &= \sum_{i=1}^{k}(E_{r}^{i} + E_{s}^{i}), \\
\mathbf {\hat t}  &= {\arg}\min  F( \mathbf{ t}),
\end{align}
\end{subequations}
where $\mathbf {\hat t}=(\mathbf{\hat t}_1^T,\mathbf{\hat t}_2^T,\cdots,\mathbf{\hat t}_k^T)^T$ is the estimated trajectory with $\mathbf{\hat t}_k$ being the estimated translation of the robot at time instant $k$.

\subsubsection{Sliding trajectory window}\label{window}
Considering that the computation in cost function \eqref{c1} will increase as the number of constrained equations increases, a sliding trajectory window is designed to ensure that the computation in \eqref{c1} can be done real-time for some low power processors. At each time instant, we only estimate the trajectory within the sliding window instead of estimating the whole trajectory shown in Fig. \ref{fig2}. If the window size is set as $N$, only $N$ latest translations within the sliding window $\mathbf{{t}}_N^k=(\mathbf{{t}}_{k-N+1}^T,\mathbf{{t}}_{k-N+2}^T,\cdots,\mathbf{{t}}_{k}^T)^T$  will be estimated. Then, the cost function \eqref{c1} becomes 
\begin{subequations}\label{sliding}
\begin{align}
F( \mathbf{{t}}_N^k) &= \sum_{i=k-N+1}^{k}(E_{r}^{i} + E_{s}^{i}), \\
\mathbf{\hat t}_N^k  &= {\arg}\min  F( {\mathbf{ t}_N^k}),
\end{align}
\end{subequations}
where $\mathbf{\hat t}_N^k=(\mathbf{{\hat t}}_{k-N+1}^T,\mathbf{{\hat t}}_{k-N+2}^T,\cdots,\mathbf{{\hat t}}_k^T)^T$ is the estimated trajectory in the sliding window. At time instant $k+1$, when the  translation $\mathbf{{t}}_{k+1}$ is added to the sliding window, the translation $\mathbf{t}_{k-N+1}$ will be removed.  Thus, the trajectory in the sliding window is updated as $\mathbf{{t}}_N^{k+1}=(\mathbf{{t}}_{k-N+2}^T,\mathbf{{t}}_{k-N+3}^T,\cdots,\mathbf{{t}}_{k+1}^T)^T$. Therefore, the number of constrained equations in \eqref{sliding} will remain as $2N$.  

\subsubsection{Outlier rejection algorithm}

UWB-based localization may be trapped in non-line of sight (NLOS) scenarios that induce measurment outliers.  
An outlier rejection algorithm is designed to reject the outliers. 
The proposed outlier rejection algorithm \eqref{outlier} requires that the mobile robot start without measurement outliers, which is easy to be satisfied. At time instant $k$, assume that the trajectory estimate is $\mathbf{\hat t}_N^k=(\mathbf{{\hat t}}_{k-N+1}^T,\mathbf{{\hat t}}_{k-N+2}^T,\cdots,\mathbf{{\hat t}}_k^T)^T$. Then at time instant $k\!+\!1$, the robot receives a range measurement ${d_{k+1}}$ from one of the fixed anchors $\mathbf{t}_{a}^{k+1}$. This range measurement ${d_{k+1}}$ is considered as an outlier if the following condition is satisfied: 
\begin{equation}\label{outlier}
|\|\mathbf { {\hat t}}_{k}-\mathbf{t}_{a}^{k+1}\|_2-d_{k+1}|> \frac{\gamma \cdot v_{\text{max}}}{f},
\end{equation}
where 
$\gamma$ is the outlier rejection parameter, and $f$ is the frequency of UWB sensor. If the inequality \eqref{outlier} is satisfied, the corresponding range constrained equation  $E_{r}^{k+1}$ and trajectory smoothness constrained equation $E_{s}^{k+1}$
will not be added to the cost function \eqref{sliding}, thus the translation $\mathbf{t}_{k+1}$ will not be estimated. On the contrary, 
if the inequality \eqref{outlier} is not satisfied, the measurement $d_{k+1}$ will be used to localize the robot, and then the estimate $\mathbf { {\hat t}}_{k+1}$ can be obtained. The implementation of the range-only based localization algorithm is given in \textbf{Algorithm} \ref{array-sum}.

\begin{algorithm}
\caption{Range-only based Localization}
\label{array-sum}
\begin{algorithmic}[1]
\State \textbf{Initialization:} \State \ \ \ \ Received $k\! \ge \! N$ range measurements at time instant $k$;
\State \ \ \ \  \hspace*{-0.2cm} $N$ latest translations in the sliding  window $\mathbf { {t}}_N^k$;
\State \ \ \ \ \hspace*{-0.2cm} $\mathbf{{t}}_N^k=(\mathbf{{t}}_{k-N+1}^T,\mathbf{{t}}_{k-N+2}^T,\cdots,\mathbf{{t}}_{k}^T)^T$;
\State \ \ \ \ \hspace*{-0.2cm} Obtaining estimated trajectory $\mathbf {{\hat t}}_N^k$ by solving  \eqref{sliding};
\State \ \ \ \ \hspace*{-0.2cm} Obtaining estimated translation $  \hat{\mathbf t}_{k}$;
\State \ \ \ \ \hspace*{-0.2cm} $k_c=1$;
\State \textbf{While} Received range measurement $d_{k+1}$ at time instant $k\!+\!1$;
\State \ \ \ \ \textbf{If} Condition \eqref{outlier} \textbf{do}
\State \ \ \ \ \ \ \ \  Rejecting range measurement $d_{k+1}$;
\State \ \ \ \ \ \ \ \ $k=k+1$; 
\State \ \ \ \ \ \ \ \ $k_c=k_c+1$;
\State \ \ \ \ \ \ \ \ \textbf{If} $k_c>\gamma$ \textbf{do}
\State \ \ \ \ \ \ \ \ \ \ \ \ Robot is trapped in extremely bad environment;
\State \ \ \ \ \ \ \ \ \ \ \ \ \hspace*{-0.2cm} Leaving the extremely bad environment;
\State \ \ \ \ \ \ \ \ \ \ \ \ \hspace*{-0.2cm} Restarting the localization;
\State \ \ \ \ \ \ \ \ \textbf{End}
\State \ \ \ \ \textbf{Else} 
\State \ \ \ \ \ \ \ \  
$\mathbf{{t}}_N^{k+1}\!=\!(\mathbf{{t}}_{k\!-\!N\!+\!2}^T,\mathbf{{t}}_{k\!-\!N\!+\!3}^T,\cdots,\mathbf{{t}}_{k\!+\!1}^T)^T$ is updated; 
\State  \ \ \ \ \ \ \ \ Adding  $E_{r}^{k+1}$ and  $E_{s}^{k+1}$ to the cost function \eqref{sliding};
\State \ \ \ \ \ \ \ \ Removing $E_{r}^{k-N+1}$ and  $E_{s}^{k-N+1}$ from  \eqref{sliding};
\State \ \ \ \ \ \ \ \  Obtaining the estimate $\mathbf { {\hat t}}_N^{k+1}$ by solving \eqref{sliding};
\State \ \ \ \ \ \ \ \ Obtaining estimated translation $\mathbf { {\hat t}}_{k+1}$;
\State \ \ \ \ \ \ \ \ $k=k+1$;
\State \ \ \ \ \ \ \ \ $k_c=1$;
\State \ \ \ \  \textbf{End}
\State   \textbf{End}
\end{algorithmic}
\end{algorithm}

\subsection{Range-orientation based localization}\label{fusion}

The proposed range-only based localization algorithm in section \ref{a44} is used to estimate the robot translation, {which can accommodate other measurements to estimate the robot pose.} In this section, an example for fusing orientation to estimate the robot pose is presented. In the experiments, the inertial measurement unit (IMU) \cite{Anonymous:uDLQ94uT} is used to verify this method. 

\subsubsection{Range constrained equation}
The robot pose at time instant $k$ is denoted as $\mathbf{P}_k \in \mathbb{R}^{4\times4}$, so that the robot trajectory can be denoted by $\mathbf{P}=(\mathbf{P}_1^T,\mathbf{P}_2^T,\cdots,\mathbf{P}_k^T)^T$. 
The range constrained equation ${E_{r}^{k}}$ in \eqref{26} can be rewritten in a form of pose:
\begin{subequations}\label{26a}
\begin{align}
{E_{r}^{k}} &= w_{r}^{k} \cdot \rho( e_{r}^{k} ),\\
 e_{r}^{k}  &=  d_{k} - ||(\mathbf {P}_{k}-\mathbf{P}_a^{k})\cdot (0,0,0,1)^T||_2,
\end{align}
\end{subequations}
where $\mathbf{P}_a^k$ is the anchor pose.

\subsubsection{Trajectory smoothness constrained equation}
The trajectory smoothness constrained equation $E_{p}^{k}$ between adjacent poses   $\mathbf{P}_k$ and $\mathbf{P}_{k-1}$ is designed as
\begin{subequations}\label{31}
\begin{align}
{E_{p}^{k}} &= \rho \left( \sqrt{  {\mathbf{e}_{p}^{k}}^{T} \mathbf {w}_p^{k}\mathbf{e}_{p}^{k} } \right),\\
\mathbf{e}_{p}^{k} &=  \mathbf{ log_{\text{SE(3)}}}(\left[\begin{array}{cc}
    \mathbf{\breve R}_{k-1}^{-1} \mathbf{\breve R}_{k}  & 0\\
    \mathbf{0}  &  1 
    \end{array}\right] \cdot \mathbf{P}_{k}^{-1}\cdot \mathbf{P}_{k-1}),
\end{align}
\end{subequations}
where function $\mathbf{ log_{\text{SE(3)}}(\cdot)}$ is defined in \eqref{21}. The $\mathbf{\breve R}_k$ and $\mathbf{\breve R}_{k-1}$ are measurements from orientation sensor. $\mathbf w_p^{k} \in \mathbb{R}^{6 \times 6}$ is the weight designed as
\begin{equation}\label{d2}
\mathbf {w}_p^{k} = \diag{\mathbf{\mathbf {w}}^{k}_{o},\mathbf{\mathbf {w}}_{t}^{k} },
\end{equation}
where $\mathbf{\mathbf {w}}^{k}_{o} \in \mathbb{R}^{3 \times 3}$ is the weight of the rotation estimate, which is provided by orientation sensor.
$\mathbf{\mathbf {w}}_{t}^{k} \in \mathbb{R}^{3 \times 3}$ is the weight of translation estimate designed as
\begin{equation}\label{diagnal}
\mathbf{w}_{t}^{k} =
\diag{w_s^k,w_s^k,w_s^k},    
\end{equation}
where $w_s^k $ is designed in \eqref{d1}. Since $\mathbf{w}^{k}_{o}$ and $\mathbf{w}^{k}_{t}$ are independent in \eqref{d2}, we can know that fusing the orientation information to \eqref{31} will not influence the estimation of translation. 
By combining \eqref{26a} and \eqref{31}, the trajectory in the sliding window $\mathbf{{P}}_N^k=(\mathbf{P}_{k-N+1}^T,\mathbf{P}_{k-N+2}^T,\cdots,\mathbf{P}_k^T)^T$ can be estimated by solving the following cost function:
\begin{subequations}\label{c2}
\begin{align}
 F(\mathbf{P}_N^k) &= \sum_{i=k-N+1}^k(E_{r}^{i} + E_{p}^{i}), \\
 \mathbf {\hat P}_N^k &= {\arg}\min  F(\mathbf {{P}}_N^k),
\end{align}
\end{subequations}
where $\mathbf {\hat P}_N^k=(\mathbf{\hat P}_{k-N+1}^T,\mathbf{\hat P}_{k-N+2}^T,\cdots,\mathbf{\hat P}_k^T)^T$ is the estimated trajectory with $\mathbf{{\hat P}}_k$ being the estimated pose of the robot at time instant $k$.

\section{Optimization }\label{a6}

\subsection{Optimization in Euclidean Space}\label{eucli}

For the range-only based localization, its cost function \eqref{sliding} with respect to $\mathbf{{t}}_N^k=(\mathbf{{t}}_{k-N+1}^T,\mathbf{{t}}_{k-N+2}^T,\cdots,\mathbf{{t}}_{k}^T)^T$ is optimized in Euclidean space. 
\begin{equation}\label{p1}
     F( \mathbf{t}_N^k) =  \sum_{i=k-N+1}^{k}\phi_r^i+\phi_s^i,
\end{equation}
where
\begin{equation}
\begin{array}{ll}
 & \phi_r^i = \! w_{r}^{i} \cdot \rho \left(d_{i} \! - \! ||\mathbf {t}_{i} \!- \! \mathbf{t}_a^{i} \! ||_2 \right), \\
 &   \phi_s^i= w_{s}^{i} \cdot \rho(   ||{\mathbf t}_{i}-{\mathbf t}_{i-1}||_2), \\
 & \rho(\cdot) = \xi^2(\sqrt{1+(\cdot/\xi)^2}-1).
\end{array}
\end{equation}

The Levenberg-Marquardt \cite{Luo:2007cz} method is used to optimize this cost function. The initial guess of $\mathbf{t}_N^k$ is denoted by $\mathbf{\bar t}_N^k$, and $\Delta \mathbf {t}_N^k =  \mathbf {t}_N^k -  \mathbf {\bar t}_N^k $ is an increment. Then, we can obtain 
\begin{equation}\label{43}
 F( \mathbf{t}_N^k) \! \simeq \!  F( \mathbf{\bar t}_N^k)\! +\! {\Delta \mathbf {t}_N^k}^T\bigtriangledown \! F( \mathbf{\bar t}_{N}^k)\! + \!\frac{1}{2}{\Delta \mathbf { t}_N^k}^{T}\mathbf{B}_{N}( \mathbf{\bar t}_{N}^k)\Delta \mathbf {t}_N^k,  
\end{equation}
where $\mathbf{B}_{N}( \mathbf{\bar t}_{N}^k)$ is a symmetric matrix that approximates the Hessian matrix $\bigtriangledown^2 \! F( \mathbf{\bar t}_{N}^k)$. 
Taking the derivative with respect to the increment $ \Delta \mathbf {t}_N^k$, we have
\begin{equation}\label{45}
\mathbf { B}_N( \mathbf{\bar t}_{N}^k) \Delta \mathbf t_N^k = -\bigtriangledown \! F( \mathbf{\bar t}_{N}^k) .
\end{equation}

Considering that the increment $\Delta \mathbf {t}_N^k$ cannot be acquired uniquely if the matrix $\mathbf B_N( \mathbf{\bar t}_{N}^k)$ is singular,
the Levenberg-Marquardt method introduces a damping factor $\lambda_k>0$ to solve this problem
\begin{equation}\label{46}
 (\mathbf {B}_N( \mathbf{\bar t}_{N}^k) + \lambda_k \mathbf {I})\Delta \mathbf t_N^k= -\bigtriangledown \! F( \mathbf{\bar t}_{N}^k),
\end{equation}
where $\mathbf{I}$ is identity matrix.
The estimate $\mathbf { \mathbf{\hat{t}}}_N^k$ is obtained by
\begin{equation}\label{47}
\begin{array}{ll}
  \mathbf{{\hat t}}_N^k = \mathbf{\bar  t}_N^k + \Delta \mathbf t_N^k, \ \
\mathbf {\bar  t}_N^k=  \mathbf{\hat t}_N^k.
\end{array}
\end{equation}

The optimization process \eqref{43}-\eqref{47} continues until the number of iterations is reached or {the value of cost function is smaller than a threshold.}

\subsection{Optimization in Non-Euclidean Space}\label{noneucli}

For fusing the orientation to estimate the robot pose, the corresponding cost function \eqref{c2} is optimized in non-Euclidean space. To solve  \eqref{c2}, an idea is to map the cost function \eqref{c2} into Euclidean space by a mapping function \eqref{21}. The pose $\mathbf{P}_k$ can be represented by its corresponding vector $\epsilon_k \in \mathbb{R}^6$, so that the cost function \eqref{c2} can be rewritten as:
\begin{subequations}\label{abc}
\begin{align}
& F(\epsilon_N^k) = \sum_{i=k-N+1}^k(E_{r}^{i} + E_{p}^{i}), \\
& \mathbf {\hat \epsilon}_N^k = {\arg}\min  F(\mathbf \epsilon_N^k),
\end{align}
\end{subequations}
where $\epsilon_N^k=(\mathbf{\epsilon}_{k-N+1}^T,\epsilon_{k-N+2}^T,\cdots,\epsilon_k^T)^T$. Similar to the process \eqref{43}-\eqref{47}, the estimate $\mathbf{\hat \epsilon}_N^k=(\hat \epsilon_{k-N+1}^T,\hat \epsilon_{k-N+2}^T,\cdots,\hat\epsilon_k^T)^T$ can be obtained. Then, the estimated trajectory $\mathbf {\hat P}_N^k=(\mathbf {\hat P}_{k-N+1}^T,\mathbf {\hat P}_{k-N+2}^T,\cdots,\mathbf {\hat P}_k^T)^T$ can be acquired by \eqref{21}.

\subsection{Computational Complexity}\label{complexity}
Observe that the computational complexity of the proposed range-only based algorithm and range-orientation based algorithm is dominated by \eqref{43}-\eqref{47}. Suppose that  $M$ is the number of iterations used for Levenberg-Marquardt method and $N$ is the size of the sliding window.
The matrix inverse calculation is involved and its complexity is around $O(N^3)$ \cite{williams2004steerable}. Hence, the computational upper bound of optimization in \eqref{43}-\eqref{47} is $O(MN^3)$. Fortunately, the matrix $\mathbf {B}_N \!+\! \lambda \mathbf {I}$ in \eqref{46} is sparse if $N>5$, i.e. only $9(3N-2)$ entries are non-zero. By taking advantage of this characteristic of $\mathbf {B}_N + \lambda \mathbf {I}$, the increment $\Delta \mathbf t_N^k$ in \eqref{47} can be calculated efficiently by the sparse Cholesky factorization algorithm \cite{chen2008algorithm}. 
When the window size is set as $N=10$ and the number of iterations is taken as $M=10$, the corresponding running time of the proposed algorithm in 
different power processors is shown in Table \ref{tablec}. 
Note that if the update rate (reciprocal of running time) of the proposed algorithm is lower than the frequency of the UWB sensor, the proposed algorithm will lose some range measurements, resulting in a larger translation estimation error. The frequency of the UWB sensor in our experiments is $32.46\mathrm{Hz}$. Hence, the maximum allowed running time is $1/32.46 \simeq 0.0308 \mathrm{s}$ .
It can be seen from Table \ref{tablec} that the proposed method can be applied in some
low power processors  such as Intel Atom x7-Z8750 and Quad-core ARM Cortex-A53 for localization because they can run the proposed algorithm in less than $0.0308 \mathrm{s}$.

\begin{table}[t]
  \begin{center}
  \caption{Running time of the proposed algorithm in different power processors with $M=10$ and $N=10$}
        \begin{tabular}{cc}
        \toprule
         \textbf{Processors} & \textbf{Running time (s)}     \\
         
        \midrule
        Intel core i7 processor &
        0.0019
        \\
        Intel Atom x7-Z8750 processor & 0.0057 \\
        Quad-core ARM Cortex-A53 processor & 0.0301 \\
        \bottomrule
        \end{tabular}
        \label{tablec}
  \end{center}
  \end{table}

\section{Stability Analysis}\label{a5}

For the range-only based cost function \eqref{p1}, the performance of the iterative algorithms \eqref{47} will be analyzed. For the range measurements $d_N^k=(d_{k-N+1},d_{k-N+2},\cdots,d_{k})$, {$\mathbf{{t}}_N^k=(\mathbf{{t}}_{k-N+1}^T,\mathbf{{t}}_{k-N+2}^T,\cdots,\mathbf{{t}}_{k}^T)^T$} is the corresponding trajectory in the sliding window. 
The translation
$\mathbf {t}_{i} = \mathbf{A}_i \cdot \mathbf{t}_{N}^k$ can be extracted from $\mathbf{t}_{N}^k$ by a matrix $\mathbf{A}_i \in \mathbb{R}^{3 \times 3N}$.  $\mathbf{\bar t}_N^k=(\mathbf{{\bar t}}_{k-N+1}^T,\mathbf{{\bar t}}_{k-N+2}^T,\cdots,\mathbf{{\bar t}}_k^T)^T$ is the initial guess of $\mathbf{ t}_N^{k}$.  
Then, the gradient and Hessian of the cost function \eqref{p1} with respect to the initial guess $\mathbf{\bar t}_{N}^k$  are 
\begin{equation}\label{gr1}
\begin{array}{ll}
     &\bigtriangledown \! F( \mathbf{\bar t}_{N}^k) \\
     &\!= \! -\sum\limits_{i=k\!-\!N\!+\!1}^{k} \! \mathbf{A}_i^T \! \cdot \!  y_i(\mathbf{ \bar t}_i) \! \cdot \! w_{r}^{i} \! \cdot \! \bigtriangledown \!  \mathbf h_i(\mathbf{\bar t}_i)  +  \\
     &  \sum\limits_{i=k\!-\!N\!+\!2}^{k} (\mathbf{A}_i^T \! - \!\mathbf{A}_{i\!-\!1}^T)  \! \cdot \! {\mathbf z}_i(\mathbf{ \bar t}_i-\mathbf{ \bar t}_{i-1})  \! \cdot \! w_s^i  \!  + \! \chi_{k\!-\!N\!+\!1}.
\end{array}
\end{equation}
where
\begin{equation}\label{gradient}
\begin{array}{ll}
    & m_i(\mathbf{ \bar t}_i) = d_{i} \! - ||\mathbf{ \bar t}_i  -  \mathbf{t}_a^{i}||_2, \\
    &
    y_i(\mathbf{ \bar t}_i) = \frac{m_i(\mathbf{ \bar t}_i)}{\sqrt{1+\frac{m_i^2(\mathbf{ \bar t}_i)}{\xi^2}}}, \\
    & \bigtriangledown \!  \mathbf h_i(\mathbf{ \bar t}_i) = \frac{\mathbf{ \bar t}_i  -  \mathbf{t}_a^{i}  }{ ||\mathbf{ \bar t}_i  -  \mathbf{t}_a^{i}||_2}, \\ 
    &  \mathbf{z}_i(\mathbf{\bar t}_i - \mathbf{ \bar t}_{i-1}) = \frac{\mathbf{\bar t}_i \!-\! \mathbf{\bar t}_{i\!-\!1}}{\sqrt{1+\frac{||\mathbf{\bar t}_i \!-\! \mathbf{\bar t}_{i\!-\!1}||_2^2}{\xi^2}}}, \\
    & \chi_{k\!-\!N\!+\!1} = 
     \mathbf{A}_{k\!-\!N\!+\!1}^T \cdot \! \frac{\mathbf{\bar t}_{k\!-\!N\!+\!1} \!-\! \mathbf{\hat t}_{k\!-\!N}}{\sqrt{1+\frac{||\mathbf{\bar t}_{k\!-\!N\!+\!1} \!-\! \mathbf{\hat t}_{k\!-\!N}||_2^2}{\xi^2}}}  \! \cdot \! w_s^{k\!-\!N\!+\!1}.
\end{array}
\end{equation}

The $\mathbf{\hat t}_{k\!-\!N}$ in \eqref{gradient} is from the estimate $\mathbf{\hat t}_N^{k\!-\!1}$,  and
\begin{equation}
\begin{array}{ll}
    &\bigtriangledown^2 \! F( \mathbf{\bar t}_{N}^k) \\
    &\!  =  \! -\sum\limits_{i=k\!-\!N\!+\!1}^{k} \! \mathbf{A}_i^T \! \cdot \! w_{r}^{i} \! \cdot \! \mathbf{s}_i(\mathbf{\bar t}_{i}) \! \cdot \! \mathbf{A}_i + \\
    &  \sum\limits_{i=k\!-\!N\!+\!2}^{k} (\mathbf{A}_i^T \!- \!\mathbf{A}_{i\!-\!1}^T) \! \cdot \! \mathbf{x}_i(\mathbf{ \bar t}_i- \mathbf{ \bar t}_{i-1}) \! \cdot \! w_s^i \! \cdot \! (\mathbf{A}_i\! -\! \mathbf{A}_{i\!-\!1}) + \\
    &   \mathbf{A}_{k\!-\!N\!+\!1}^T \! \cdot \!  \varpi_{k\!-\!N\!+\!1} \cdot \! w_s^{k\!-\!N\!+\!1} \! \cdot \! \mathbf{A}_{k\!-\!N\!+\!1},
\end{array}
\end{equation}
where
\begin{equation}
\begin{array}{ll}
   & \mathbf{s}_i(\mathbf{\bar t}_{i}) \!=  \bigtriangledown \! \mathbf y_i(\mathbf{\bar t}_{i}) \cdot \bigtriangledown \! \mathbf h_i^T(\mathbf{\bar t}_{i}) \!+ \! y_i(\mathbf{\bar t}_{i}) \cdot \bigtriangledown^2 \! \mathbf h_i(\mathbf{\bar t}_{i}), \\
    & \bigtriangledown \! \mathbf y_i(\mathbf{\bar t}_{i}) = - l_i(\mathbf{\bar t}_{i}) \cdot  \bigtriangledown \!  \mathbf h_i(\mathbf{\bar t}_{i}), \\
    & l_i(\mathbf{\bar t}_{i}) = 
    \frac{1}{(1+\frac{m_i^2(\mathbf{\bar t}_{i})}{\xi^2})^{\frac{3}{2}}}, \\
    & \bigtriangledown^2 \! \mathbf h_i(\mathbf{\bar t}_{i}) = \frac{\mathbf{I}-\bigtriangledown \! \mathbf h_i(\mathbf{\bar t}_{i}) \cdot \bigtriangledown \! \mathbf h_i^T(\mathbf{\bar t}_{i})}{||\mathbf{t}_i-\mathbf{t}_a^i||_2}, \\
    & \mathbf{x}_i(\mathbf{\bar t}_{i}-\mathbf{ \bar t}_{i-1}) =   \frac{(\xi^2+||\mathbf{\bar t}_{i}-\mathbf{\bar t}_{i-1}||_2^2)\mathbf{I} - (\mathbf{\bar t}_{i}-\mathbf{\bar t}_{i-1})(\mathbf{\bar t}_{i}-\mathbf{\bar t}_{i-1})^T}{\xi^2(1+\frac{||\mathbf{\bar t}_{i}-\mathbf{\bar t}_{i-1}||_2^2}{\xi^2})^{\frac{3}{2}}}, \\
    & \varpi_{k\!-\!N\!+\!1} = \frac{(\xi^2+||\mathbf{\bar t}_{k\!-\!N\!+\!1}-\mathbf{\hat t}_{k\!-\!N}||_2^2)\mathbf{I}- (\mathbf{\bar t}_{k\!-\!N\!+\!1}-\mathbf{\hat t}_{k\!-\!N})(\mathbf{\bar t}_{k\!-\!N\!+\!1}-\mathbf{\hat t}_{k\!-\!N})^T}{\xi^2(1+\frac{||\mathbf{\bar t}_{k\!-\!N\!+\!1}-\mathbf{\hat t}_{k\!-\!N}||_2^2}{\xi^2})^{\frac{3}{2}}}.
\end{array}
\end{equation}

From \eqref{25} and \eqref{m2}, we can know that $\mathbf{t}_N^k$ belongs to the set $X_N^k$:
\begin{equation}\label{set}
\begin{array}{ll}
 X_N^k =\{\hspace*{-0.3cm}&\mathbf{t}_N^k : ||\mathbf{t}_i\!-\!\mathbf{t}_a^i||_2\le d_i\!+\!\eta,||\mathbf{t}_{i}\!-\!\mathbf{t}_{i-1}||_2 \le v_{\max}\cdot T_N^k, \\
 &T_N^k=\max (\Delta T_{k-N+1},\Delta T_{k\!-N\!+\!2},\cdots,\Delta T_{k} ), \\ 
 &i=k\!-\!N\!+\!2,k\!-\!N\!+\!3,\cdots,k \}. 
\end{array}
\end{equation}

For the Hessian matrix, let 
\begin{equation}\label{bb1}
    \delta_{s}^{k} = \min_{\ \{ \mathbf{t}_N^{k} \in X_N^k, \theta \in [0,1] \}} ||\bigtriangledown^2 \! F(\theta \mathbf{t}_N^{k}\!+\!(1\!-\!\theta)\mathbf{\bar t}_N^{k})||_2,
\end{equation}
and 
\begin{equation}\label{bb2}
    \delta_{l}^{k} = \max_{\ \{ \mathbf{t}_N^{k} \in X_N^k, \theta \in [0,1] \}} ||\bigtriangledown^2 \! F(\theta \mathbf{t}_N^{k}\!+\!(1\!-\!\theta)\mathbf{\bar t}_N^{k})||_2.
\end{equation}

Also, let 
\begin{equation}\label{bb3}
    \mu_{N}^{k} =  ||\mathbf{B}_N(\mathbf{\bar{t}}_N^k)+ \lambda_k \mathbf{I}||_2,
\end{equation}
where $\mathbf{B}_N(\mathbf{\bar{t}}_N^k)$ is from \eqref{43}.
Inspired by the work \cite{alessandri2017fast}, 
define stability parameter $\alpha_k$ as
\begin{equation}\label{es}
\alpha_k \! = \! \max (|1\!-\!\frac{\delta_s^{k}}{\mu_N^{k}}|,|1\!-\!\frac{\delta_l^{k}}{\mu_N^{k}}|).  
\end{equation}

The stability parameter $\alpha_k$ plays a crucial role in ensuring
the convergence of the proposed method.

\begin{theorem}
Suppose stability parameter \eqref{es} satisfies $\alpha_i <1, (i=N\!+\!1,N\!+\!2,\cdots,k\!+\!1)$. Then, the trajectory estimation error for the range-only based cost function \eqref{p1} given by the Levenberg-Marquardt method 
is bounded if the range measurement noise is bounded.
\end{theorem}

\begin{proof}
The initial guess $\mathbf{\bar t}_N^{k+1}$ of $\mathbf{t}_N^{k+1}$ {is set as $\mathbf{\bar t}_N^{k+1}= \mathbf{\hat t}_N^{k}$. }
The estimate $\mathbf{\hat t}_N^{k+1}$  based on the \eqref{46} and \eqref{47} with single iteration is
\begin{equation} \label{p3}
    \mathbf{\hat t}_N^{k+1}= \mathbf{\bar t}_N^{k+1} - (\mathbf{B}_N(\mathbf{\bar t}_N^{k+1}) + \lambda_{k+1} \mathbf{I})^{-1}  \bigtriangledown \! F( \mathbf{\bar t}_N^{k\!+1}).
\end{equation}

According to the mean value theorem, it follows that there exists $\theta \in [0,1]$ such that
\begin{equation}\label{p2}
\begin{array}{ll}
    & \bigtriangledown \! F( \mathbf{\bar t}_N^{k\!+1}) =   \bigtriangledown \! F( \mathbf{t}_N^{k\!+1}) + \\
    & \bigtriangledown^2 \! F(\theta \mathbf{t}_N^{k\!+1}+(1-\theta)\mathbf{\bar t}_N^{k\!+1})  (\mathbf{\bar t}_N^{k\!+1}-\mathbf{ t}_N^{k\!+1}).
\end{array}
\end{equation}

Then, combining \eqref{gr1} and \eqref{25}, it yields
\begin{equation}\label{p4}
\begin{array}{ll}
    & \bigtriangledown \! F( \mathbf{\bar t}_N^{k\!+1}) = \eta_N^{k+1} + \\
    & \bigtriangledown^2 \! F(\theta \mathbf{t}_N^{k\!+1}+(1-\theta)\mathbf{\bar t}_N^{k\!+1})(\mathbf{\bar t}_N^{k\!+1}-\mathbf{ t}_N^{k\!+1}),
\end{array}  
\end{equation}
where 
\begin{equation}
\begin{array}{ll}
   & \eta_N^{k+1} \\
   & =  -\sum\limits_{i=k-N+2}^{k+1} \! \mathbf{A}_i^T \cdot y_i(\mathbf{ t}_{i}) \cdot w_{r}^{i} \cdot \bigtriangledown \!  \mathbf h_i(\mathbf{ t}_{i}) + \\
   &  \sum\limits_{i=k-N+3}^{k+1}(\mathbf{A}_i^T \! - \!\mathbf{A}_{i\!-\!1}^T)  \! \cdot \! \mathbf{z}_i(\mathbf{  t}_i \! -\!  \mathbf{ t}_{i\!-\!1})  \! \cdot \! w_s^i   + \chi_{k\!-\!N\!+\!2}. 
\end{array}
\end{equation}

From \eqref{fa} and \eqref{gradient}, we have
\begin{equation}\label{bound}
\begin{array}{ll}
    &y_i(\cdot)<\xi, ||\mathbf{z}_i(\cdot)||_2<\xi, ||w_{r}^{i}||_2<1, ||w_{s}^{i}||_2<1,\\
    &||\bigtriangledown \! \mathbf{h}_i(\cdot)||_2=1, 
    ||\mathbf{A}_i||_2 \le 1, ||\mathbf{A}_i \! - \!\mathbf{A}_{i\!-\!1}||_2 \le 2, \\
    & ||\chi_{k-N+2}||_2 \le  \xi.
\end{array}
\end{equation}

Then, 
it yields 
\begin{equation}
\begin{array}{ll}
&||\eta_N^{k\!+\!1}||_2 \! <   N \xi + 2 (N-1) \xi + \xi \le (3 N -1
)\xi.
\end{array}
\end{equation}

Let trajectory estimation error be $\mathbf{e}_{k+1}=\mathbf{\hat t}_N^{k+1}-\mathbf{ t}_N^{k+1}$.  $\mathbf{r}_N^{k+1}\!=\!(\mathbf{r}_{k-N+2}^T,\mathbf{r}_{k\!-N+3}^T,\cdots,\mathbf{r}_{k+1}^T)^T$
is the unknown relative translations with $||\mathbf{r}_N^{k+1}||_2 \le N \! \cdot \! v_{\max} \! \cdot \! T_N^{k+1} $.  We have $\mathbf{\bar t}_N^{k+1}-\mathbf{t}_N^{k+1}=\mathbf{\hat t}_N^{k}-\mathbf{t}_N^{k}-\mathbf{r}_N^{k+1}=\mathbf{e}_{k} - \mathbf{r}_N^{k+1}$.

Combining \eqref{p3} with \eqref{p4} leads to
\begin{equation}
\begin{array}{ll}
    & \mathbf{e}_{k+1}  = \mathbf{\bar t}_N^{k\!+\!1}\!-\!\mathbf{t}_N^{k\!+\!1} \!-\! (\mathbf{B}_N(\mathbf{\bar t}_N^{k\!+\!1})\! +\! \lambda_{k\!+\!1} \mathbf{I})^{-1}  \bigtriangledown \! F( \mathbf{\bar t}_N^{k\!+1}) \\
    & =   - (\mathbf{B}_N(\mathbf{\bar t}_N^{k\!+\!1}) \! + \! \lambda_{k\!+\!1} \mathbf{I})^{-1}\eta_N^{k+1} + ( \mathbf{e}_{k}-\mathbf{r}_N^{k\!+\!1}) \cdot \\
    & \left (1\!-\!(\mathbf{B}_N(\mathbf{\bar t}_N^{k\!+\!1}) \! + \! \lambda_{k\!+\!1} \mathbf{I})^{-1}\bigtriangledown^2 \! F(\theta \mathbf{t}_N^{k\!+1}\!+\!(1\!-\!\theta)\mathbf{\bar t}_N^{k\!+\!1}) \right).
\end{array}  
\end{equation}

Therefore,
\begin{equation}
 ||\mathbf{e}_{k\!+\!1}||_2 \! \le \! \frac{||\eta_N^{k\!+\!1}||_2}{\mu_N^{k\!+\!1}}\!+ \!
   \max (|1\!-\!\frac{\delta_s^{k\!+\!1}}{\mu_N^{k\!+\!1}}|,|1\!-\!\frac{\delta_l^{k\!+\!1}}{\mu_N^{k\!+\!1}}|)(||\mathbf{e}_{k}||_2\!+\! ||\mathbf{r}_N^{k\!+\!1}||_2). 
\end{equation}

Let 
\begin{equation}\label{small}
 \begin{array}{ll}
 &  \alpha_i \! = \! \max (|1\!-\!\frac{\delta_s^{i}}{\mu_N^{i}}|,|1\!-\!\frac{\delta_l^{i}}{\mu_N^{i}}|), \beta_i \! = \! \frac{||\eta
 _N^{i}||_2}{\mu_N^{i}},   c_i \! = \! N \cdot v_{\max} \cdot T_N^i ,\\
 &(i=N\!+\!1,N\!+\!2,\cdots,k\!+\!1), \\
 & \alpha= \max (\alpha_{N\!+\!1},\alpha_{N\!+\!2},\cdots, \alpha_{k\!+\!1}), \\
 & \beta= \max (\beta_{N+1},\beta_{N+2},\cdots, \beta_{k+1}), \\
 & c = \max(c_{N+1},c_{N+2},\cdots,c_{k+1}).
 \end{array}
\end{equation}

Then, we have
\begin{equation}\label{size}
    ||\mathbf{e}_{k\!+\!1}||_2 \le  
    \alpha^{k-N+1}  ||\mathbf{e}_{N}||_2 \! +  \!
      \sum\limits_{i=0}^{k-N} \alpha^{i}\beta \! + \! \sum\limits_{i=0}^{k-N} \alpha^{i+1}c.   
\end{equation}

If condition $\alpha_i <1, (i=N\!+\!1,N\!+\!2,\cdots,k\!+\!1)$ holds and {the range measurement noise is bounded}, we get
\begin{equation}
\lim\limits_{ k \rightarrow \infty}  ||\mathbf{e}_{k+1}||_2  \le   \frac{\beta}{1-\alpha} + \frac{\alpha c}{1-\alpha}. 
\end{equation}

The above single iteration result can  be extended to multi-iteration scheme.  {For multi-iteration scheme, it is easy to prove that the trajectory estimation error is also bounded if the range measurement noise is bounded
and $\alpha_i <1, (i=N\!+\!1,N\!+\!2,\cdots,k\!+\!1)$.}
\end{proof}

\begin{remark}
The proposed method does not require accurate initial guess.
{The parameter $\lambda_k$ in \eqref{bb3} can be adjusted 
to ensure $\alpha_k < 1$ \eqref{es}.}
Although the convergence of the cost function \eqref{p1} can be guaranteed by choosing suitable parameter $\lambda_k$,  we can not guarantee that there exists a unique solution to the cost function \eqref{p1} because the cost function \eqref{p1} is a non-convex function. From our experiment experience, the local minimum is more likely to be avoided by setting a larger window size $N$. 
\end{remark}

Since $\mathbf{w}^{k}_{o}$ and $\mathbf{w}^{k}_{t}$ are independent in \eqref{d2}, we can know that fusing the orientation information to the proposed framework will not influence the convergence of the translation estimation.

\begin{figure*}[t]
	\centering
	\includegraphics[width=1\linewidth]{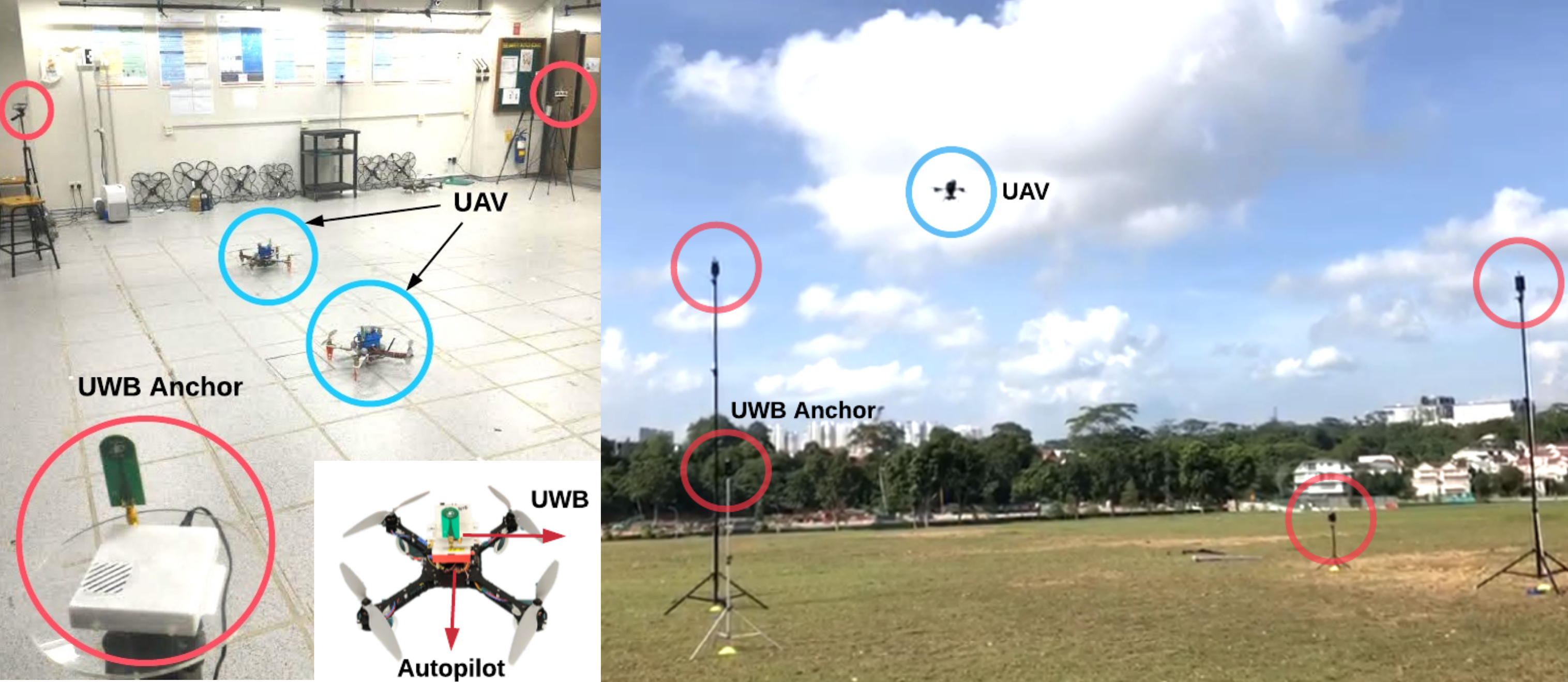}
	\caption{The experiments are carried out in both indoor and outdoor environments. The localization system in the experiments consists of only four fixed anchors. The UAV can send requests to four fixed UWB anchors sequentially to get the range measurements.}
	\label{fig6}
\end{figure*}

\section{Experimental Results}\label{a7}

In this section, we use a UAV to test the performance of the proposed range-only based and range-orientation based localization algorithms. The hardware, experiment setup, evaluation, parameter selection, comparison of localization accuracy, the effects of different number of iterations and window size in the optimization on the localization accuracy, analysis of robust and smoothness of the proposed algorithm are presented, respectively.

\subsection{Hardware}

Quadrotor UAV consists of four rotors which are configured in a cross shape as shown in Fig. \ref{fig6}. The UAV 
is equipped with UWB module, IMU sensor and Pixhawk Autopilot.   
The UWB platform for the experiment is from Time Domain with operating band from 3.1 GHz to
5.3 GHz shown in Fig. \ref{fig6}. Within the the range of $100 \mathrm{m}$, it is able to provide
precise measurement at an update rate of around $40 \mathrm{Hz}$. Its
dimension ($7.6 \mathrm{cm} \times 8.0 \mathrm{cm} \times 1.6 \mathrm{cm}$) and weight ($58 \mathrm{g}$) are
suitable for micro unmanned aerial vehicles.  Our UWB ranging algorithm uses two-way TOA measurement to calculate the range shown in \eqref{25}, which is able to provide relatively steady range measurements and has the ranging area of $100 \mathrm{m}$ with ranging error within $\eta=0.2 \mathrm{m}$. The IMU sensor is from myAHRS+ (altitude heading reference system), which is a low cost high performance attitude and heading reference system (AHRS) containing a 3-axis 16-bit gyroscope, a 3-axis 16-bit accelerometer, and a 3-axis 13-bit magnetometer.

\subsection{Experiment Setup and Evaluation}
The indoor experiments were carried out in an area of $6 \mathrm{m} \times 6 \mathrm{m} \times 3 \mathrm{m}$, and outdoor experiments were carried out within an area of $6 \mathrm{m} \times 8 \mathrm{m} \times 5 \mathrm{m}$.
The four non-coplanar fixed anchors are used in our experiments. UAV will send requests to four fixed UWB anchors sequentially to get the range measurements. The ground truth is provided by a VICON system which has a localization accuracy of mm-level. In order to analyze the localization accuracy, {the mean error $\rm E_{T}$ and root mean square error $\rm E_{RMSE}$ of translation are given by}
\begin{equation}
   \rm E_{T} = \frac{1}{k}\sum_{i=1}^{k}||\mathbf {\hat t}_i - \mathbf t_{i}||_2,
\end{equation}
\begin{equation}\label{1123}
   \rm E_{RMSE} =\sqrt{ \frac{1}{k}\sum_{i=1}^{k}||\mathbf {\hat t}_i - \mathbf t_{i}||_2^2},
\end{equation}
where $\mathbf {\hat t}_i$ and $\mathbf t_{i}$ are the estimate and ground truth, respectively. 

The mean error of rotation $\rm E_{O}$ is given by
\begin{equation}\label{rotation}
    \rm E_{O} = \frac{1}{k}\sum_{i=1}^{k}||\mathbf{\hat R}_i\mathbf{R}_{i}^{-1}- \mathbf{I}||.
\end{equation}
where $\mathbf{\hat R}_{i}$ is the estimated rotation, and $\mathbf{ R}_{i}$ is the rotation from the orientation sensor.

\begin{table*}[t]
  \begin{center}
  \caption{ Comparison of existing range-based localization methods in 2-D plane}
        \begin{tabular}{cccccccccc}
        \toprule
         \textbf{2-D plane} & \textbf{Proposed range-only based Method}    & \textbf{MHE} & \textbf{UKF}  &  \textbf{EKF}  & \textbf{NR} & \textbf{Particle Filter} & \textbf{GRNN} & \textbf{BPNN} & \textbf{KNN} \\
         
        \midrule
        Translation Error (m) &
        {\textbf{0.031}} & 0.078  & 0.091 & 0.102 & 0.154 &0.115 & 0.114 & 0.128 &0.166
        \\
        \bottomrule
        \end{tabular}
        \label{table1}
  \end{center}
  \ \ \ \ \ \ \ \ \ \ \ \ \ \ \ \ \ \ \ \
  Best result is highlighted in black boldface. The \textbf{RMSE} of the proposed method in 2-D plane is \textbf{0.033}m.

  \ \ \ \ \ \ \ \ \ \ \ \ \ \ \ \ \ \ \ \
  The range-only based localization and range-orientation based localization achieved the same localization accuracy in translation.
    
  \ \ \ \ \ \ \ \ \ \ \ \ \ \ \ \ \ \ \ \  \textbf{GRNN}: Generalized regression neural network,  
  \textbf{BPNN}: Back propagation neural network, 
  \textbf{KNN}: K-nearest neighbor method.
\end{table*}

\begin{table*}[t]
  \begin{center}
  \caption{Comparison of existing range-based localization methods in 3-D space}
        \begin{tabular}{cccccccc}
        \toprule
     \textbf{3-D space}  &\textbf{Proposed range-only based Method}  & \textbf{UKF} &  \textbf{EKF} & \textbf{NR} & \textbf{RVFL+FS} & \textbf{MDS+PSO} & \textbf{MHE}\\
        \midrule
         Mean Error of $x$ (m)  & {\textbf{0.023}}  &0.062 & 0.114 &  0.102 & - & - & -\\
        Mean Error of $y$ (m) &{\textbf{0.022}}   &0.066 & 0.123  &  0.116 & - & - & -\\
        Mean Error of $z$ (m) &{\textbf{0.077}}  &\textcolor{red}{\textbf{0.232}} & \textcolor{red}{\textbf{0.353}} &  \textcolor{red}{\textbf{0.346}} & - & - & -\\
        Translation Error (m) 
       &{\textbf{0.083}}  & \textcolor{red}{\textbf{0.249}} & \textcolor{red}{\textbf{0.391}} & \textcolor{red}{\textbf{0.379}} & \textcolor{red}{\textbf{0.340}} &\textcolor{red}{\textbf{0.698}}  &\textcolor{red}{\textbf{\textgreater 1}}
        \\
        \bottomrule
        \end{tabular}
        \label{table2}
  \end{center}
  \ \ \ \ \ \ \ \ \ \ \ \ \ \ \ \ \ \ \ \
  Best results are highlighted in black boldface. The \textbf{RMSE} of the proposed method in 3-D space is \textbf{0.081}m. 
  
  \ \ \ \ \ \ \ \ \ \ \ \ \ \ \ \ \ \ \ \
  Weak results in the existing range-based localization methods are highlighted in red boldface.
  
  \ \ \ \ \ \ \ \ \ \ \ \ \ \ \ \ \ \ \ \
  The range-only based localization and range-orientation based localization achieved the same localization accuracy in translation.
  
  \ \ \ \ \ \ \ \ \ \ \ \ \ \ \ \ \ \ \ \  \textbf{PSO}: Particle swarm optimization, 
  \textbf{RVFL}:  Random vector functional link network,
  \textbf{FS}: Feature Selection.
\end{table*}

Since IMU and UWB measurements may not be obtained at the same time instant, 
a time synchronizer filter is used to synchronize incoming IMU and UWB measurements. At each time instant when a UWB measurement is received, we get a ranging measurement for constructing range constrained equation \eqref{26a}, then IMU measurement whose time instant is closest to the incoming UWB time instant is chosen for constructing trajectory smoothness constrained equation \eqref{31}. Since the IMU has a much higher data rate, the remaining IMU measurements are not used.

\subsection{Parameter Selection}\label{para}

The system parameters are shown as follows.


\begin{enumerate}
\item Anchor translations $\mathbf{t}_a^k$ in \eqref{26} in indoor environments are $(3,3,1.95)$, $(3,-3,0.53)$, $(-3,3,0.54)$ and $(-3,-3,1.98)$, respectively, and $(0,0,0.79)$, $(6,0,5)$, $(6,8,1.52)$ and $(0,8,5.52)$ in outdoor environments.
The unit is metre$(\mathrm{m})$;
\item Range measurement $d_k$ in \eqref{26} is from UWB sensor;
\item The upper bound of range measurement noise is $\eta=0.2 \mathrm{m}$;
\item Frequency of UWB sensor in \eqref{outlier} is $f =32.46\mathrm{Hz}$;
\item Frequency of IMU sensor is
$100.3\mathrm{Hz}$;

\item  The variance of range measurement noise in \eqref{fa} is set as $\sigma_r^2= \frac{\eta^2}{9}$;
\item The selections of the weights $w_s^{k}$ and  $\mathbf{w}^{k}_{t}$ are presented in \eqref{d1} and \eqref{diagnal};
\item The initial guess of robot rotation $\mathbf{R}_k$ in $\mathbf{P}_k$ is set as $\mathbf{\breve R}_k$, where $\mathbf{\breve R}_k$ is the IMU measurement. The initial guess of robot translation is set randomly.
\end{enumerate}

\begin{figure}[t]
\centering
\includegraphics[width=0.95\linewidth]{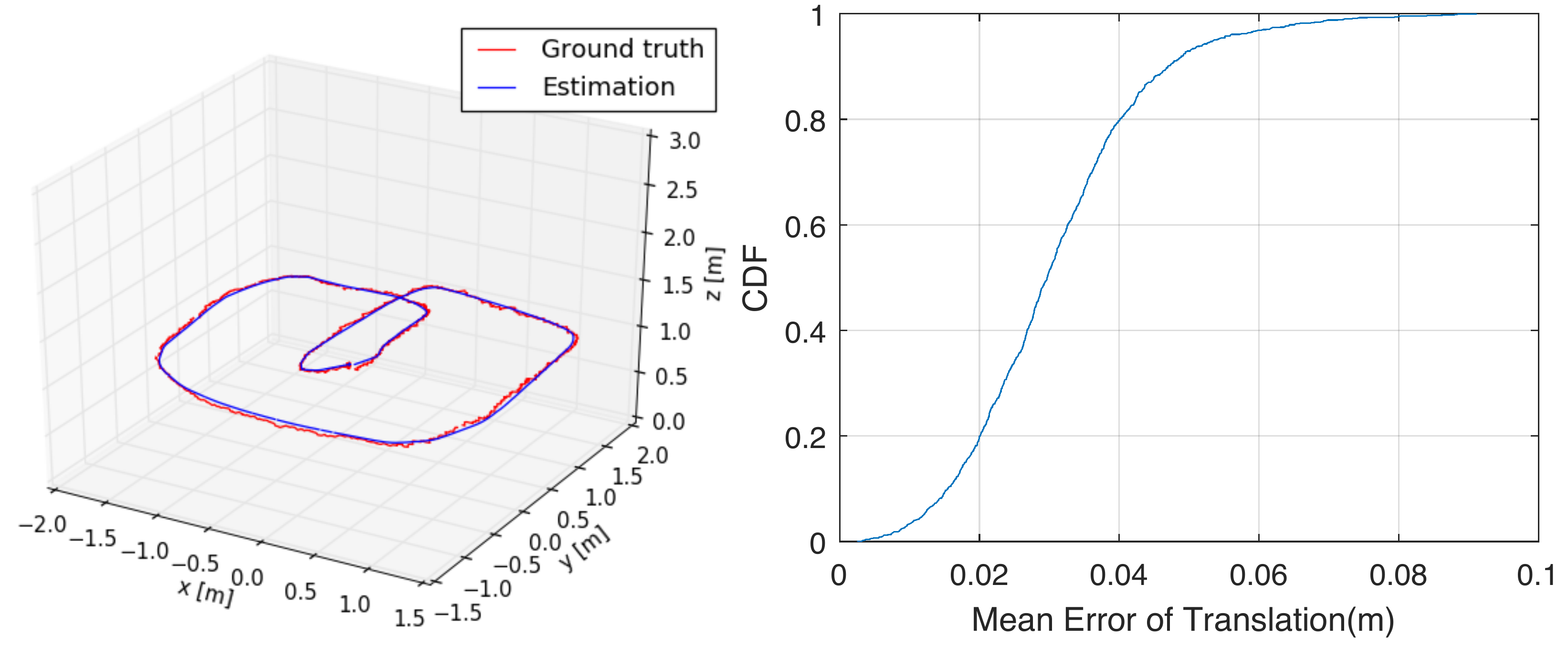}
\caption{UAV moved along a circle or rectangle in a 2-D plane.  In the left figure,
the red line is Ground truth from VICON system, and the blue line is the range-only based estimation. The cumulative distribution of mean error of translation is presented in the right figure.}
\label{fig8}
\end{figure}

\subsection{Comparison with Existing Range-based Localization Methods}

The UAV moved along a circle or a rectangle in our experiments. The range-only based localization \eqref{sliding} and range-orientation based localization \eqref{c2} obtained similar localization accuracy in translation, which verified that fusing the orientation information to the proposed framework will not influence the estimation of translation.  For the range-orientation based localization \eqref{c2},
the mean error of rotation in \eqref{rotation} is only $0.0023$.

\begin{figure}[t]
\centering
\includegraphics[width=0.95\linewidth]{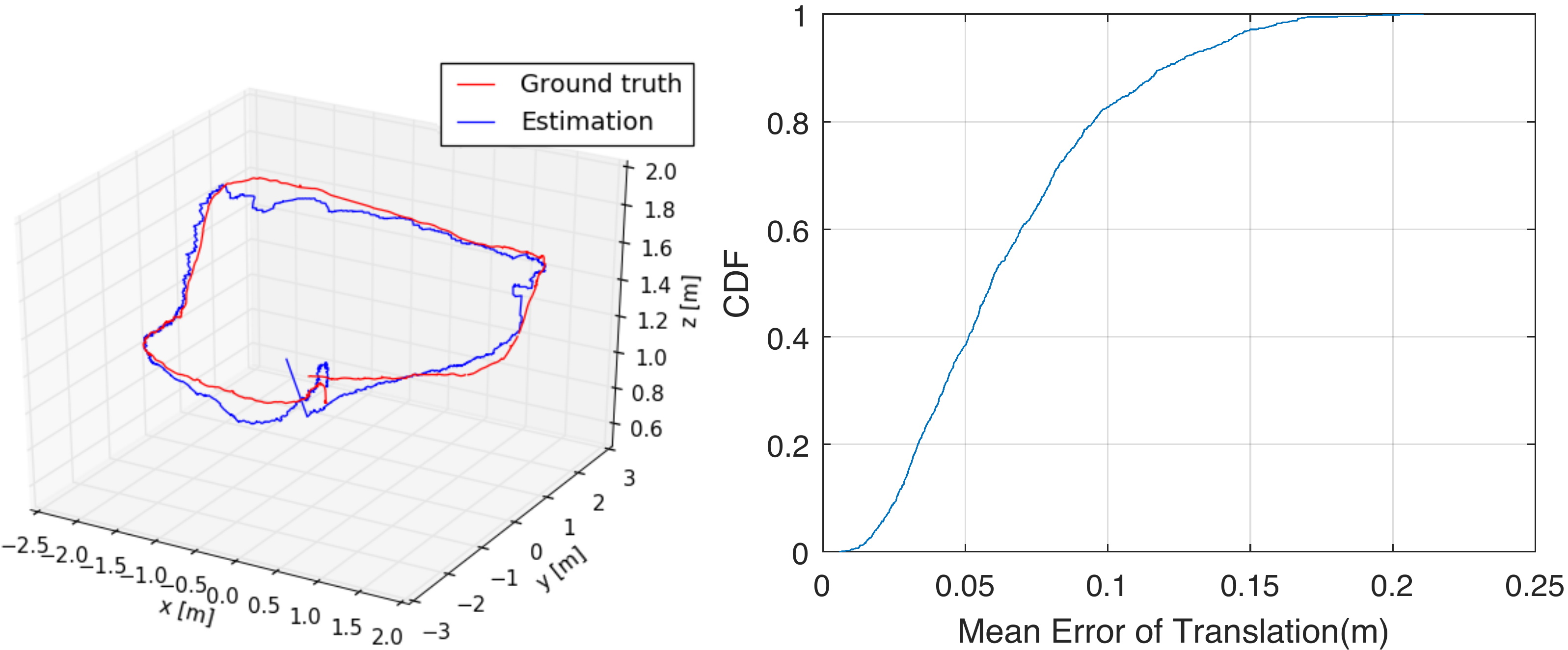}
\centering
\caption{ UAV moved in a 3-D space. In the left figure, the red line is ground truth from VICON system, and the blue line is the range-only based estimation. The cumulative distribution of mean error of translation is presented in the right figure. }
\label{fig9}
\end{figure}

For the comparison with existing methods (UKF, EKF, nonlinear regression (NR), MDS, RVFL)\cite{Guo:2016ff,Anonymous:tWWRLPVI, Cui:2016bg,cui2018received}, we implemented their algorithms in the same environment of our lab with the best choices of parameters to ensure the fairness of the comparison. The results of moving horizon estimation and particle filter method are from \cite{Prorok:2014jc, girrbach2017optimization, pillonetto2010unconstrained}.
We adopted the experimental results of the machine learning methods (GRNN, BPNN and KNN) in the latest work \cite{chen2018grof}.

We conduct more than 50 experiments in 2-D plane, in which one of the examples is shown in Fig. \ref{fig8}, where the cumulative distribution (CDF) of mean error of translation is presented, and the ground truth and estimated trajectory are shown in red and blue lines, respectively. 
The mean error and root mean square error of translation of the proposed method are 3.1 and 3.3 centimeters, respectively. The CDF of mean error of  translation in 2-D plane shows that about 75 percent of translation errors are within 4 centimeters.
In comparison with other range-based localization methods in 2-D plane, the localization accuracy of the proposed algorithm has improved more than 3 times as shown in Table  \ref{table1}.

\begin{figure}[t]
\centering
\includegraphics[width=1\linewidth]{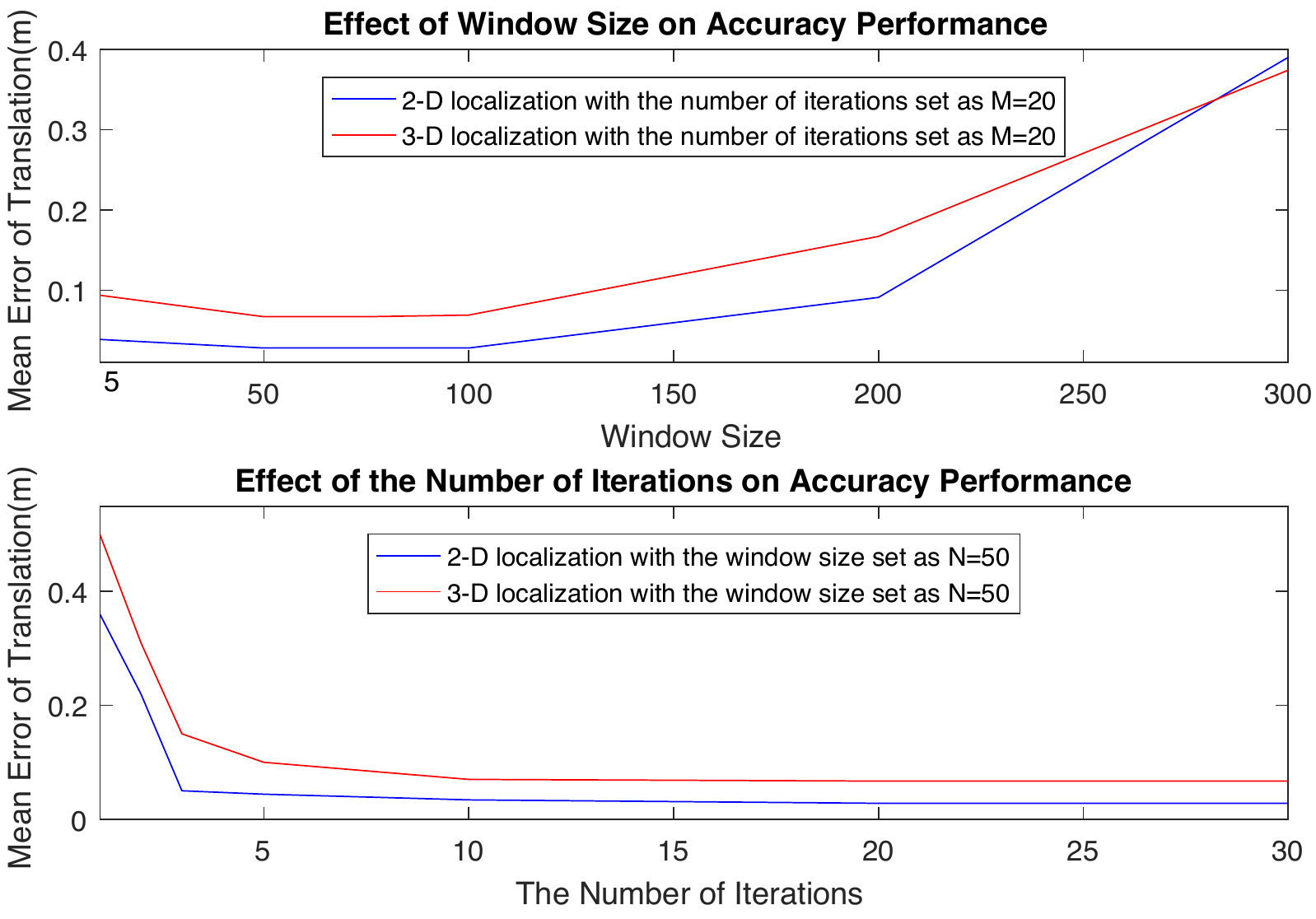}
\caption{Effects of different iteration number and window size on mean translation error. }
\label{figap}
\end{figure}

Similarly, Fig. \ref{fig9}. shows one of the experiments conducted in 3-D space. The mean error and root mean square error of translation of the proposed method are 8.3 and 8.1 centimeters, respectively.
The CDF of mean error of translation in 3-D space shows that about 75 percent of translation errors are within 8 centimeters. The comparison with other range-based localization methods in 3-D space is shown in Table \ref{table2}. The localization accuracy in $x$ or $y$ direction has improved more than 3 times. The mean errors of the existing methods are about 6.2 to 12.3 centimeters in the $x$ and $y$ directions. But the mean errors of our proposed algorithm are about 2.3 and 2.2 centimeters in the $x$ and $y$ directions, which are much more accurate than the existing methods. 

Compared with existing experimental results, 
it is worth noting that the localization accuracy in the altitude improved greatly  without the need of
adding altitude sensors to measure the altitude or placing the anchors on the ceiling.  The mean error in the $z$ direction is about 7.7 centimeters, which is accurate enough to fly a UAV in 3-D space as demonstrated in our UAV flight experiment.

\begin{remark}
As the ground truth of the mobile robot is obtained by a VICON system which is limited to an area of $ 6 \mathrm{m} \times 6 \mathrm{m} \times 3 \mathrm{m}$, we can only compare the proposed method with others within the area of $ 6 \mathrm{m} \times 6 \mathrm{m} \times 3 \mathrm{m}$. It is worth noting that the proposed general framework was also successfully applied in the UWB-aided visual SLAM with an improved performance \cite{wang2017ultra}. 
\end{remark}

\subsection{The effects of the Number of Iterations and Window Size }\label{parameter}

There are two important tuning parameters needed to be analyzed. One is the number of iterations $M$, and the other is the window size $N$.  The effects of different number of iterations and window size in the optimization on localization accuracy are analyzed on 2.2GHz intel core i7 processor.
The mean error of translation of the proposed method with different number of iterations and window size  in both 2-D plane and 3-D space are shown in Fig. \ref{figap}. When $N>5 $ and $M \cdot N<2000$, good localization accuracy can be obtained.

It is intuitive that if the update rate of proposed algorithm (reciprocal of running time) is lower than the frequency of UWB sensor, the proposed algorithm will lose some range measurements, resulting in a larger translation estimation error.
To guarantee a good performance of the proposed method on different power processors and scenarios, the following two conditions should be satisfied: $(a)$
The mobile robot can receive the range measurements from at least four non-coplanar anchors; 
$(b)$ The number of iterations and window size should be adjusted to ensure that the update rate of the proposed algorithm is larger than the frequency of the range sensor.  

\begin{figure}[t]
\centering
\includegraphics[width=1\linewidth]{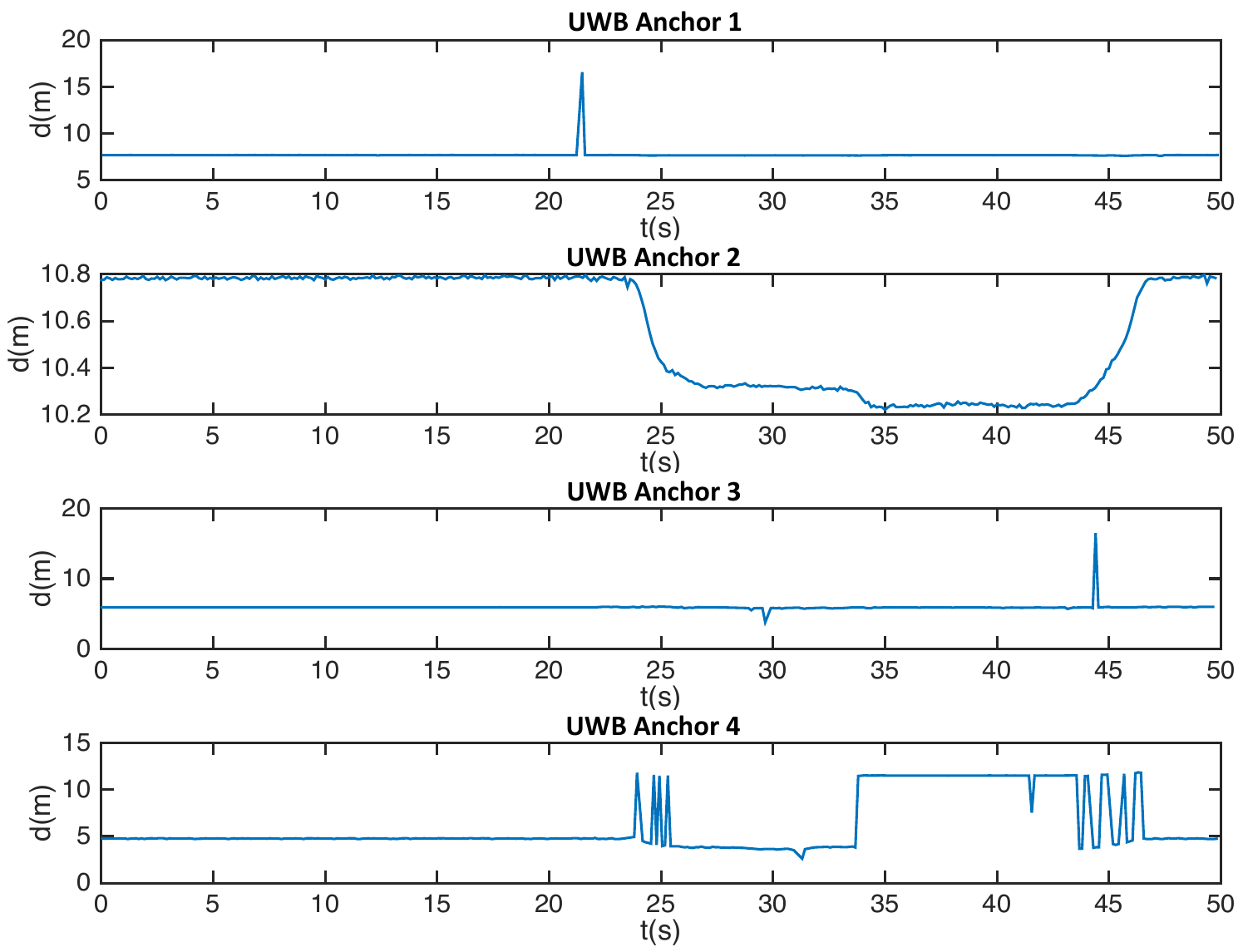}
\caption{Existence of outliers in the measurements.}
\label{figa5}
\end{figure}

\subsection{Robust:Outlier Rejection}

An outlier rejection test with four fixed UWB anchors and a static robot is presented.
The range measurements from the four anchors should be almost constant if there is no NLOS measurement. In this test, UWB anchors have outliers, which are rejected based on the algorithm \eqref{outlier}. Our method can still provide stable translation estimation even if the outliers exist for seconds. For example, UWB Anchor 4 has outliers from $23.7$ second to $25.4$ second and $33.8$ second to $47.3$ second, respectively 
shown in Fig. \ref{figa5}.
The translations of four fixed UWB anchors are $(0, 0, 0.77)$, $(6.13, 0, 5)$, $(6.01, 8.07, 0.79)$ and $(0.11, 8.02, 5.02)$, respectively. The test result showed that we can get accurate translation estimation $(0.14, 7.63, 0.29)$, which is close to the ground truth $(0.12, 7.61, 0.23)$ during the whole $50$ seconds, which verifies the robust of the proposed method.

\subsection{Smoothness:
Comparison with Barometer Sensor}

We have verified that the proposed method obtains high localization accuracy in the altitude. In addition, we find that the proposed method performs better on smoothness than the barometer. It can be seen from Fig. \ref{figa1} that barometer reading is noisy, but our proposed method is smooth due to the design of trajectory smoothness constrained equation between adjacent translations.

\section{Conclusion}\label{a9}

In this paper, a general graph optimization based localization framework was proposed which allows fusion of various sensor measurements for localization. Special emphasis was then given to the range-based localization, which removes the dependence on kinematic model and requirement of receiving multiple range measurements concurrently, and can be implemented real-time in some low power systems. Compared with existing range-based localization methods, better localization accuracy in both 2-D plane and 3-D space were obtained, especially in the altitude direction.

\ifCLASSOPTIONcaptionsoff
  \newpage
\fi

\begin{figure}[t]
 \centering
\includegraphics[width=1\linewidth]{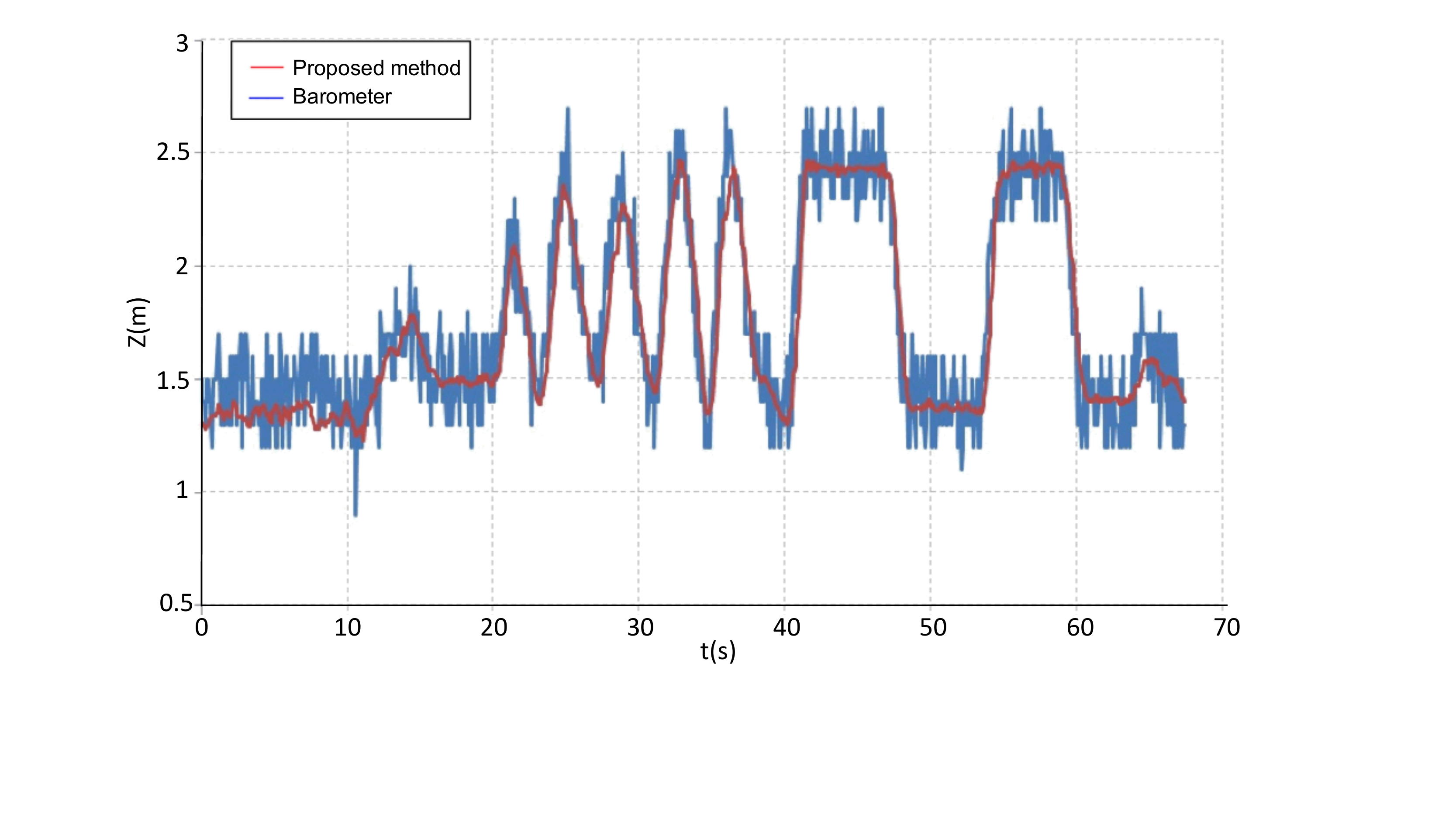}
\caption{The comparison between barometer sensor (blue line) and our method (red line) in estimating altitude.}
\label{figa1}
\end{figure}

\bibliographystyle{IEEEtran}
\bibliography{papers}

\vfill

\end{document}